\documentclass{article}

\usepackage[final,nonatbib]{neurips_2019}

\usepackage[utf8]{inputenc} 
\usepackage[T1]{fontenc}    
\usepackage[hidelinks]{hyperref}       
\usepackage{url}            
\usepackage{booktabs}       
\usepackage{nicefrac}       
\usepackage{microtype}      
\usepackage{amsmath}  
\usepackage{amsfonts} %
\usepackage{amssymb}  %
\usepackage{amsthm,bm}
\usepackage{mathtools}
\usepackage{dsfont}
\usepackage{multirow}
\usepackage{thm-restate}
\usepackage{graphicx}
\usepackage{todonotes}
\usepackage{algorithm}
\usepackage{algorithmic}
\usepackage{wrapfig}
\usepackage{enumitem}

\makeatletter
\newcommand{\printfnsymbol}[1]{%
  \textsuperscript{\@fnsymbol{#1}}%
}
\makeatother

\DeclareMathOperator*{\E}{{\mathbb{E}}}

\newcommand{\paraspace}{\vspace{-2.5mm}}
\newcommand{\figref}[1]{Figure~\ref{fig:#1}}
\newcommand{\tabref}[1]{Table~\ref{tab:#1}}
\newcommand{\secref}[1]{Sec.~\ref{sec:#1}}
\newcommand{\thmref}[1]{Theorem~\ref{thm:#1}}
\newcommand{\lemref}[1]{Lemma~\ref{lem:#1}}
\newcommand{\appref}[1]{Appendix~\ref{sec:#1}}
\newcommand{\algref}[1]{Algorithm~\ref{alg:#1}}
\newcommand{\eq}[1]{\eqref{eq:#1}}

\newlength\savedwidth
\newcommand\whline[1]{\noalign{\global\savedwidth\arrayrulewidth
                               \global\arrayrulewidth #1} %
                      \hline
                      \noalign{\global\arrayrulewidth\savedwidth}}

\title{Rethinking Generative Mode Coverage:\\ A Pointwise Guaranteed Approach}

\author{%
  Peilin Zhong\thanks{equal contribution} 
  \qquad
  Yuchen Mo\printfnsymbol{1} 
  \qquad
  Chang Xiao\printfnsymbol{1} 
  \qquad
  Pengyu Chen 
  \qquad
  Changxi Zheng \vspace{1.5mm}\\
  Columbia University \\
  \texttt{\{peilin, chang, cxz\}@cs.columbia.edu} \\
  \texttt{\{yuchen.mo, pengyu.chen\}@columbia.edu}\\
}

\begin{document}

\maketitle

\vspace{-5mm}
\begin{abstract}
\vspace{-2.5mm}
Many generative models have to combat \emph{missing modes}.
The conventional wisdom to this end is by reducing through training
a statistical distance (such as $f$-divergence) between the generated distribution
and provided data distribution.
But this is more of a heuristic than a guarantee.
The statistical distance measures a \emph{global}, but not \emph{local}, similarity 
between two distributions. Even if it is small,  
it does not imply a plausible mode coverage.
Rethinking this problem from a game-theoretic perspective, 
we show that a complete mode coverage is firmly attainable.
If a generative model can approximate a data distribution moderately well under
a global statistical distance measure, then we will be able to find a mixture
of generators that collectively covers \emph{every} data point and thus
\emph{every} mode, with a lower-bounded generation probability. 
Constructing the generator mixture has a connection to the multiplicative weights 
update rule, upon which we propose our algorithm.
We prove that our algorithm guarantees complete mode coverage.
And our experiments on real and synthetic datasets confirm
better mode coverage over recent approaches, ones that also use generator mixtures 
but rely on global statistical distances.
\end{abstract}

\section{Introduction}\label{sec:intro}




A major pillar of machine learning, the \emph{generative} approach aims at
learning a data distribution from a provided training dataset.  While
strikingly successful, many generative models suffer from \emph{missing modes}.
Even after a painstaking training process, the generated samples represent only a
limited subset of the modes in the target data distribution,
yielding a much lower entropy distribution. 

Behind the missing mode problem is the conventional wisdom of training a generative model.
Formulated as an optimization problem, the training process reduces a statistical 
distance between the generated distribution and the target data distribution. 
The statistical distance, such as $f$-divergence or Wasserstein distance, is often a
\emph{global} measure. It evaluates an integral of the discrepancy between 
two distributions over the data space (or a summation over a discrete dataset). 
In practice, reducing the global statistical distance to a perfect zero 
is virtually a mission impossible. 
Yet a small statistical distance does not certify the generator complete mode coverage.
The generator may neglect underrepresented modes---ones that are less frequent in data space---in
exchange for better matching the distribution of well represented modes,
thereby lowering the statistical distance. 
In short, a global statistical distance is not ideal for promoting mode coverage
(see \figref{motivation} for a 1D motivating example and later \figref{sin_comp} for examples of a few classic generative models).



This inherent limitation 
is evident in various types of generative models (see \appref{other_types} for the analysis of a few classic generative models). 
Particularly in generative adversarial networks (GANs), mode collapse 
has been known as a prominent issue.
Despite a number of recent improvements 
toward alleviating it~\cite{salimans2016improved,metz2016unrolled,srivastava2014dropout,x2018bourgan,chen2016infogan,bowman2016generating}, 
none of them offers a complete mode coverage.
In fact, even the fundamental question remains unanswered: \emph{what precisely does a complete mode coverage mean?}
After all, the definition of ``modes'' in a 
dataset is rather vague, depending on what specific distance metric is used
for clustering data items (as discussed and illustrated in~\cite{x2018bourgan}).

\begin{wrapfigure}[23]{r}{0.468\textwidth}
  \centering
  \vspace{-5mm}
  \includegraphics[width=0.468\textwidth]{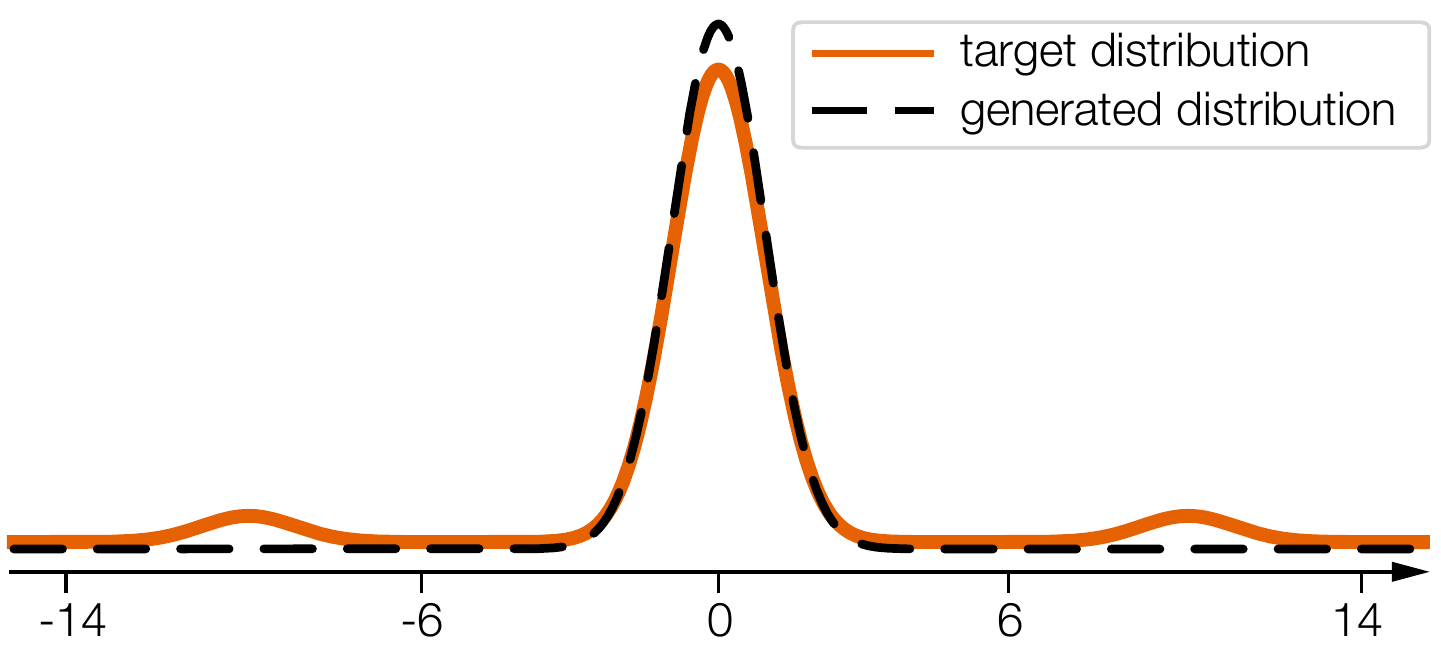}
  \vspace{-7.5mm}
  \caption{ \textbf{Motivating example.} 
Consider a 1D target distribution $P$ with three modes, i.e., a mixture of three Gaussians,
$P = 0.9\cdot\mathcal{N}(0,1)+0.05\cdot\mathcal{N}(10,1)+0.05\cdot\mathcal{N}(-10,1)$
(solid orange curve).
If we learn this distribution using a single Gaussian $Q$ (black dashed curve).
The statistical distance between the two is small: $D_{\textrm{TV}}(Q\parallel P)\leq 0.1$
and $D_{\textrm{KL}}(Q\parallel P)\leq 0.16$.
The probability of drawing samples from the side modes (in $[-14,-6]$ and $[6,14]$) of the target distribution 
$P$ is $\Pr_{x\sim P}[6\leq |x|\leq 14]\approx 0.1$, but
the probability of generating samples from $Q$ in the same intervals is
$\Pr_{x\sim Q}[6\leq |x|\leq 14]\approx 10^{-9}$. The side modes are missed!}
\label{fig:motivation}
\end{wrapfigure}
We introduce an explicit notion of complete mode coverage, 
by switching from the global statistical distance
to local \emph{pointwise} coverage:  provided a target
data distribution $P$ with a probability density $p(x)$ at each point $x$ of
the data space $\mathcal{X}$, we claim that a generator $G$ has a \emph{complete mode coverage} of $P$
if the generator's probability $g(x)$ for generating $x$ is pointwise lower bounded, that is,
\begin{equation}\label{eq:pointwise}
    g(x) \ge \psi\cdot p(x), \forall x\in\mathcal{X},
\end{equation}
for a reasonably large relaxation constant $\psi\in(0,1)$. 
This notion of mode coverage ensures that \emph{every point} $x$ in the data space $\mathcal{X}$
will be generated by $G$ with a finite and lower-bounded probability $g(x)$.
Thereby, in contrast to the generator trained by reducing a global statistical
distance (recall \figref{motivation}), no mode will have an arbitrarily small generation 
probability, and thus no mode will be missed.
Meanwhile, our mode coverage notion~\eq{pointwise} stays compatible with the conventional heuristic toward reducing a global 
statistical distance, as the satisfaction of \eq{pointwise} 
implies that the total variation distance between $P$ and $G$ is upper bounded by $1-\psi$
(see a proof in Appendix~\ref{sec:local_compatible_global}).

At first sight, the pointwise condition~\eq{pointwise} seems more stringent than reducing a global statistical distance, and
pursuing it might require a new formulation of generative models.
Perhaps somewhat surprisingly, a rethink from a game-theoretic perspective
reveal that this notion of mode coverage is viable without formulating any new models.
Indeed, a mixture of existing generative models (such as GANs) suffices. 
In this work, we provide an algorithm for constructing the generator mixture and a theoretical 
analysis showing the guarantee of our mode coverage notion~\eq{pointwise}.

\subsection{A Game-Theoretic Analysis}\label{sec:game}
Before delving into our algorithm, we offer an intuitive view of \emph{why 
our mode coverage notion~\eq{pointwise} is attainable} through a game-theoretic lens.
Consider a two-player game between Alice and Bob:  
given a target data distribution $P$ and a family $\mathcal{G}$ of
generators\footnote{An example of the generator family is the GANs.  The definition will be made clear later in this paper.},
Alice chooses a generator $G\in\mathcal{G}$, 
and Bob chooses a data point
$x\in\mathcal{X}$. If the probability density $g(x)$ of Alice's $G$ generating Bob's choice of $x$
satisfies $g(x)\ge\frac{1}{4}p(x)$,
the game produces a value
$v(G,x)=1$, otherwise it produces $v(G,x)=0$.
Here $\nicefrac{1}{4}$ is used purposely as an example to concretize our intuition.
Alice's goal is to maximize the game value, while Bob's goal is to minimize the game value.

Now, consider two situations. In the first situation, Bob first chooses a mixed strategy, that
is, a distribution $Q$ over $\mathcal{X}$. 
Then, Alice chooses the best generator $G\in\mathcal{G}$ according to Bob's
distribution $Q$. When the game starts, Bob samples a point $x$ using his choice of distribution $Q$.
Together with Alice's choice $G$, the game produces a value. 
Since $x$ is now a random variable over $Q$,
the expected game value is $\max_{G\in\mathcal{G}}\underset{x\sim Q}{\E}[v(G,x)]$.  
In the second situation, Alice first chooses a mixed strategy, that is,
a distribution $R_\mathcal{G}$ of generators over $\mathcal{G}$.
Then, given Alice's choice $R_\mathcal{G}$,
Bob chooses the best data point $x\in\mathcal{X}$.
When the game starts, Alice samples a generator $G$ from the chosen distribution $R_\mathcal{G}$.
Together with Bob's choice of $x$, the game produces a value, and the expected value
is $\min_{x\in \mathcal{X}}{\E}_{G\sim R_{\mathcal{G}}}[v(G,x)]$.

According to von Neumann's minimax theorem~\cite{neumann1928theorie, du2013minimax}, 
Bob's optimal expected value in the first situation must be the same as
Alice's optimal value in the second situation:
\begin{equation}\label{eq:minimax}
\min_Q \max_{G\in\mathcal{G}}\underset{x\sim Q}{\E}[v(G,x)]=
\max_{R_\mathcal{G}}\min_{x\in \mathcal{X}}\underset{G\sim R_{\mathcal{G}}}{\E}[v(G,x)].
\end{equation} 
With this equality realized,
our agenda in the rest of the analysis is as follows. First, we show a lower bound of the
left-hand side of~\eq{minimax}, and then we use the right-hand side to reach
the lower-bound of $g(x)$ as in~\eq{pointwise}, for Alice's generator $G$.
To this end, we need to depart off from the current game-theoretic analysis and discuss
the properties of existing generative models for a moment.

Existing generative models such as
GANs~\cite{goodfellow2014generative,salimans2016improved,arjovsky2017wasserstein}
aim to reproduce arbitrary data distributions. 
While it remains intractable to have the generated distribution
match \emph{exactly} the data distribution,
the approximations are often plausible.
One reason behind the plausible performance is that the data space 
encountered in practice is ``natural'' and 
restricted---all English sentences or
all natural object images or all images on a manifold---but not a space of arbitrary data.  
Therefore, it is reasonable to expect the generators 
in $\mathcal{G}$ (e.g., all GANs) to meet the following requirement%
\footnote{This requirement is weaker than the mainstream goal of generative models,
which all aim to approximate a target data distribution as closely as possible.
Here we only require the approximation error is upper bounded.}
(without conflicting the no-free-lunch theorem~\cite{wolpert1997no}): 
for any distribution $Q$ over a natural data space $\mathcal{X}$ encountered in practice, there exists a generator $G\in\mathcal{G}$
such that the total variation distance between $G$ and $Q$ is upper bounded by a constant 
$\gamma$, that is, $\frac{1}{2}\int_{\mathcal{X}}\left|q(x)-g(x)\right|\mathrm{d} x \leq
\gamma$, where $q(\cdot)$ and $g(\cdot)$ are the probability densities on $Q$ and the generated samples of $G$,
respectively. Again as a concrete example, we use $\gamma=0.1$.
With this property in mind, we now go back to our game-theoretic analysis.

Back to the first situation described above.
Once Bob's distribution $Q$ (over $\mathcal{X}$) and Alice's generator $G$ are
identified, then given a target distribution $P$ over $\mathcal{X}$ and an $x$
drawn by Bob from $Q$, 
the probability of having Alice's $G$ cover $P$ (i.e., $g(x)\ge\frac{1}{4}p(x)$)  
at $x$ is lower bounded.
In our current example, we have the following lower bound:
\begin{equation}\label{eq:yao_example}
\Pr_{x\sim Q}\left[g(x)\geq 1/4\cdot p(x)\right]\geq 0.4. 
\end{equation}
Here $0.4$ is related to the total variation distance bound (i.e., $\gamma=0.1$)
between $G$ and $Q$, and this lower bound value is derived in \appref{proof_game}. 
Next, notice that on the left-hand side of~\eq{minimax}, the expected value,
${\E}_{x\sim Q}[v(G,x)]$, is equivalent to the probability in~\eq{yao_example}.
Thus, we have
\begin{equation}
\min_Q \max_{G\in\mathcal{G}}\underset{x\sim Q}{\E}[v(G,x)]\ge 0.4.
\end{equation}
Because of the equality in~\eq{minimax}, this is also the lower bound of its right-hand side,
from which we know that there exists a distribution $R_\mathcal{G}$
of generators such that for any $x\in\mathcal{X}$,
we have
\begin{equation}
\underset{G\sim R_{\mathcal{G}}}{\E}[v(G,x)]=
\Pr_{G\sim R_\mathcal{G}}\left[g(x)\geq 1/4\cdot p(x)\right] \geq 0.4.
\end{equation}
This expression shows that for any $x\in\mathcal{X}$, if we draw
a generator $G$ from $R_\mathcal{G}$, then with a probability at least $0.4$, 
$G$'s generation probability density satisfies $g(x)\ge\frac{1}{4}p(x)$. Thus, we can think 
$R_\mathcal{G}$ as a ``collective'' generator $\bm{G}^*$, or a \emph{mixture of generators}. When 
generating a sample $x$, we first choose a generator $G$ according to $R_\mathcal{G}$ and then
sample an $x$ using $G$. The overall probability $g^*(x)$ of generating $x$ 
satisfies $g^*(x)>0.1p(x)$---precisely the pointwise lower bound that we pose
in~\eq{pointwise}.

\paraspace
\paragraph{Takeaway from the analysis.}
This analysis reveals that a complete mode coverage is firmly viable.
Yet it offers no recipe on \emph{how} to construct the mixture of generators 
and their distribution $R_\mathcal{G}$ using existing generative models.
Interestingly, as pointed out by Arora et al.~\cite{arora2012multiplicative},
a constructive version of von Neumann's minimax theorem is related to
the general idea of multiplicative weights update. 
Therefore, our key contributions in this work are \textbf{i)} the design of a multiplicative 
weights update algorithm (in \secref{alg}) to construct a generator mixture,
and \textbf{ii)} a theoretical analysis showing that our generator mixture
indeed obtains the pointwise data coverage~\eq{pointwise}.
In fact, we only need a small number of generators to construct the mixture (i.e., it is easy to train),
and the distribution $R_\mathcal{G}$ for using the mixture is as simple as a uniform
distribution (i.e., it is easy to use).

\section{Related Work}\label{sec:rel}

There exists a rich set of works improving classic generative models for
alleviating missing modes, especially in the framework of GANs, by altering objective 
functions~\cite{che2016mode,zhao2016energy,mao2017least,arjovsky2017wasserstein,gulrajani2017improved,saatci2017bayesian},
changing training methods~\cite{heusel2017gans, brock2018large}, 
modifying neural network architectures~\cite{metz2016unrolled,dumoulin2016adversarially,lin2017pacgan,srivastava2017veegan,karras2017progressive},
or regularizing latent space distributions~\cite{x2018bourgan,NIPS2018_7846}.
The general philosophy behind these improvements is to reduce the statistical
distance between the generated distribution and target distribution by making
the models easier to train.  
Despite their technical differences, their optimization goals are all toward reducing a global
statistical distance.  

The idea of constructing a mixture of generators has been explored, with two ways of construction.
In the first way, a set of generators are trained simultaneously. For example, 
Locatello et al.~\cite{locatello2018clustering} used multiple generators,
each responsible for sampling a subset of data points decided in a 
k-means clustering fashion.
Other methods focus on the use of multiple
GANs~\cite{arora2017generalization,hoang2018mgan,park2018megan}.
The theoretical intuition behind these approaches is by viewing a GAN
as a two-player game and extending it to reach a Nash equilibrium 
with a mixture of generators~\cite{arora2017generalization}.
In contrast, our method does not depend specifically on GANs,
and our game-theoretic view is fundamentally different (recall \secref{game}).

Another way of training a mixture of generators takes a sequential approach. 
This is related to \emph{boosting} algorithms in machine learning.
Grnarova et al.~\cite{grnarova2017online} viewed the problem of training 
GANs as finding a mixed strategy in a zero-sum game, and used
the Follow-the-Regularized-Leader algorithm~\cite{hazan2016introduction}
for training a mixture of generators iteratively.
Inspired by AdaBoost~\cite{freund1997decision}, 
other approaches train a ``weak'' generator that fits a reweighted data
distribution in each iteration, and all iterations together form an additive mixture of
generators~\cite{wang2016ensembles,tolstikhin2017adagan} or a multiplicative
mixture of generators~\cite{grover2018boosted}. 

Our method can be also viewed as a boosting strategy.  From this perspective, the most related 
is AdaGAN~\cite{tolstikhin2017adagan}, while significant differences exist.
Theoretically, AdaGAN (and other boosting-like algorithms) is based on the assumption that
the reweighted data distribution in each iteration becomes progressively easier to learn.
It requires a generator in each iteration to have a statistical distance to the reweighted
distribution smaller than the previous iteration. As we will discuss in~\secref{experiment}, this
assumption is not always feasible.  We have no such assumption. Our method can use
a weak generator in each iteration.  If the generator is more expressive, the theoretical
lower bound of our pointwise coverage becomes larger (i.e., a larger $\psi$ in~\eq{pointwise}).
Algorithmically, our reweighting scheme is simple and different from AdaGAN, only doubling
the weights or leaving them unchanged in each iteration. 
Also, in our mixture of generators, they are treated uniformly, and no mixture 
weights are needed, whereas AdaGAN needs a set of weights that are heuristically chosen.

To summarize,
in stark contrast to all prior methods, our approach is rooted in a different
philosophy of training generative models. 
Rather than striving for reducing a global statistical distance,
our method revolves around an explicit notion of complete mode coverage as defined in~\eq{pointwise}.
Unlike other boosting algorithms, our algorithm of constructing the mixture of generators
guarantees complete mode coverage, and this guarantee
is theoretically proved.



\vspace{-0.05in}
\section{Algorithm}\label{sec:alg}
\vspace{-0.05in}
\paragraph{A mixture of generators.} 
Provided a target distribution $P$ on a data domain $\mathcal{X}$, 
we train a mixture of generators to pursue pointwise mode coverage~\eq{pointwise}.
Let $\bm{G}^*=\{G_1,\dots,G_T\}$ denote the resulting mixture of $T$ generators.
Each of them ($G_t,t=1...T$) may use any existing generative model such as GANs.
Existing methods that also rely on a mixture of generators associate each generator a nonuniform weight $\alpha_t$
and choose a generator for producing a sample randomly based on the weights.
Often, these weights are chosen heuristically, e.g., in AdaGAN~\cite{tolstikhin2017adagan}.
Our mixture is conceptually and computationally simpler. Each generator is treated equally.
When using $\bm{G}^*$ to generate a sample, we first choose a generator $G_i$ uniformly
at random, and then use $G_i$ to generate the sample.

\paraspace
\vspace{-0.5mm}
\paragraph{Algorithm overview.}
Our algorithm of training $\bm{G}^*$ can be understood as a specific rule design in 
the framework of multiplicative weights update~\cite{arora2012multiplicative}. 
Outlined in \algref{mwu_training}, it runs iteratively.
In each iteration, a generator $G_t$ is trained using an updated data distribution $P_t$
(see Line 6-7 of \algref{mwu_training}).
The intuition here is simple: if in certain data domain regions the current generator fails
to cover the target distribution sufficiently well, then we update the data distribution to emphasize
those regions for the next round of generator training (see Line 9 of \algref{mwu_training}).
In this way, each generator can focus on the data distribution in individual data regions.
Collectively, they are able to cover the distribution over the entire data domain,
and thus guarantee pointwise data coverage.

\paraspace
\vspace{-0.5mm}
\paragraph{Training.}
Each iteration of our algorithm trains an individual generator $G_t$,
for which many existing generative models, such as GANs~\cite{goodfellow2014generative}, 
can be used. The only
prerequisite is that $G_t$ needs to be trained to approximate the data
distribution $P_t$ moderately well. 
This requirement arises from our 
game-theoretic analysis (\secref{game}), wherein the total variation distance
between $G_t$'s distribution and $P_t$ needs to be upper bounded. 
Later in our theoretical analysis (\secref{analysis}),
we will formally state this requirement, 
which, in practice, is easily satisfied by most existing generative models.

\begin{algorithm}[t]
    \caption{Constructing a mixture of generators}
    \label{alg:mwu_training}
    \begin{algorithmic}[1]
        \STATE {\bfseries Parameters:} $T$, a positive integer number of generators, 
        and $\delta\in(0,1)$, a covering threshold. 
        \STATE {\bfseries Input:} a target distribution $P$ on a data domain $\mathcal{X}$.
        \STATE For each $x\in \mathcal{X}$, initialize its weight $w_1(x)=p(x)$.  
        \FOR{$t = 1\rightarrow T$}
        \STATE Construct a distribution $P_t$ over $\mathcal{X}$ as follows:
        \STATE For every $x\in\mathcal{X}$, normalize the probability density $p_t(x)=\frac{w_t(x)}{W_t}$,
        where $W_t=\int_{\mathcal{X}} w_t(x) \mathrm{d}x$.
        \STATE Train a generative model $G_t$ on the distribution $P_t$. \label{sta:train_generator}
        \STATE Estimate generated density $g_t(x)$ for every $x\in\mathcal{X}$. 
        \STATE For each $x\in \mathcal{X}$, if $g_t(x)<\delta\cdot p(x)$, set $w_{t+1}(x) = 
        2\cdot w_t(x)$. Otherwise, set $w_{t+1}(x)=w_t(x)$.
        \ENDFOR
        \STATE {\bfseries Output:} a mixture of generators $\bm{G}^*=\{G_1,\ldots,G_T\}$.
    \end{algorithmic}
\end{algorithm}

\paraspace
\vspace{-0.5mm}
\paragraph{Estimation of generated probability density.}
In Line 8 of \algref{mwu_training}, we need to estimate the probability $g_t(x)$ of the 
current generator sampling a data point $x$. 
Our estimation follows the idea of adversarial training, similar to AdaGAN~\cite{tolstikhin2017adagan}. 
First, we train a discriminator $D_t$ to distinguish between samples from $P_t$ and 
samples from $G_t$. The optimization objective of $D_t$ is defined as
\begin{equation*}
\max_{D_t} \E_{x\sim P_t}[\log D_t(x)] + \E_{x\sim G_t}[\log(1-D_t(x))].
\end{equation*}
Unlike AdaGAN~\cite{tolstikhin2017adagan}, here $P_t$ is the currently updated data distribution,
not the original target distribution, and $G_t$ is the generator trained in the current
round, not a mixture of generators in all past rounds. 
As pointed out previously~\cite{nowozin2016f,tolstikhin2017adagan}, once $D_t$ is optimized,
we have $D_t(x) = \frac{p_t(x)}{p_t(x) + g_t(x)}$ for all $x\in\mathcal{X}$, and
equivalently $\frac{g_t(x)}{p_t(x)} = \frac{1}{D_t(x)} - 1$. 
Using this property in Line 9 of \algref{mwu_training} (for testing the data
coverage), we rewrite the condition $g_t(x)<\delta\cdot p(x)$ as
\begin{equation*}
    \frac{g_t(x)}{p(x)} = \frac{g_t(x)}{p_t(x)}\frac{p_t(x)}{p(x)}
    = \left(\frac{1}{D_t(x)} - 1\right)\frac{w_t(x)}{p(x)W_t} < \delta,
\end{equation*}
where the second equality utilize the evaluation of $p_t(x)$ in Line 6 
(i.e., $p_t(x) = \nicefrac{w_t(x)}{W_t}$).

Note that if the generators $G_t$ are GANs, then the discriminator of each $G_t$ can be reused 
as $D_t$ here. Reusing $D_t$ introduces no additional computation. In contrast, 
AdaGAN~\cite{tolstikhin2017adagan} always has to train an additional discriminator 
$D_t$ in each round using the mixture of generators of all past rounds.

\paraspace
\paragraph{Working with empirical dataset.}
In practice, the true data distribution $P$ is often unknown
when an empirical dataset $\mathbb{X}=\{x_i\}_{i=1}^n$ is given.
Instead, the empirical dataset 
is considered as $n$ i.i.d. samples drawn from $P$.
According to the Glivenko-Cantelli theorem \cite{cantelli1933sulla}, 
the uniform distribution over $n$ i.i.d. samples from $P$
will converge to $P$ as $n$ approaches to infinity. 
Therefore, provided the empirical dataset, we do not need to know the probability density $p(x)$ of $P$,
as every sample $x_i\in\mathbb{X}$ is considered to have a finite and uniform probability measure.
An empirical version of \algref{mwu_training} 
and more explanation are presented in the supplementary document (\algref{empirical} and 
\appref{alg_emp}).




\vspace{-0.05in}
\section{Theoretical Analysis}\label{sec:analysis}
\vspace{-0.05in}
We now provide a theoretical understanding of our algorithm, 
showing that the pointwise data coverage~\eq{pointwise} is indeed obtained.
Our analysis also sheds some light on how to choose the parameters of \algref{mwu_training}.

\vspace{-0.04in}
\subsection{Preliminaries}\label{sec:pre}
\vspace{-0.04in}
We first clarify a few notational conventions and introduce two new theoretical notions for our
subsequent analysis.
Our analysis is in continuous setting; results on discrete datasets follow 
directly.

\paraspace
\paragraph{Notation.}
Formally, we consider a $d$-dimensional measurable space
$\left(\mathcal{X},\mathcal{B}(\mathcal{X})\right)$, where $\mathcal{X}$ is
the $d$-dimensional data space, and $\mathcal{B}(\mathcal{X})$ is the Borel
$\sigma$-algebra over $\mathcal{X}$ to enable probability measure.
We use a capital letter (e.g., $P$) to denote a probability measure on 
this space. 
When there is no ambiguity, we also refer them as probability distributions 
(or distributions).
For any subset $\mathcal{S}\in\mathcal{B}(\mathcal{X})$, the probability 
of $\mathcal{S}$ under $P$ is $P(\mathcal{S})\coloneqq \Pr_{x\sim P}[x\in \mathcal{S}]$.
We use $G$ to denote a generator. When there is no ambiguity, $G$ 
also denotes the distribution of its generated samples.
All distributions are assumed absolutely continuous. 
Their probability density functions (i.e., the derivative with respect to the Lebesgue
measure) are referred by their corresponding lowercase letters (e.g.,
$p(\cdot)$, $q(\cdot)$, and $g(\cdot)$). 

Moreover, we use $[n]$ to denote the set $\{1,2,...,n\}$,
$\mathbb{N}_{>0}$ for the set of all positive integers,
and $\mathds{1}( \mathcal{E})$ 
for the indicator function whose value is $1$ if the event $\mathcal{E}$ happens, and 
$0$ otherwise.

\paraspace
\paragraph{$f$-divergence.}
Widely used in objective functions of training generative models,
$f$-divergence is a statistical distance between two distributions. 
Let $P$ and $Q$ be two distributions over $\mathcal{X}$.
Provided a convex function $f$ on $(0,\infty)$ such that $f(1) = 0$, 
$f$-divergence of $Q$ from $P$ is defined as
$
D_f(Q\parallel P)\coloneqq\int_{\mathcal{X}} f\left(\frac{q(x)}{p(x)}\right) p(x) \mathrm{d}x.
$
Various choices of $f$ lead to some commonly used 
$f$-divergence metrics such as total variation distance $D_{\textrm{TV}}$,
Kullback-Leibler divergence $D_{\textrm{KL}}$, Hellinger distance $D_{\textrm{H}}$,
and Jensen-Shannon divergence $D_{\textrm{JS}}$~\cite{nowozin2016f,amari2016information}.
Among them, total variation distance is upper bounded by many other $f$-divergences.
For instance, $D_{\textrm{TV}}(Q\parallel P)$ is upper bounded by
$\sqrt{\frac{1}{2}D_{\textrm{KL}}(Q\parallel P)}$, 
$\sqrt{2}D_{\textrm{H}}(Q\parallel P),$ and $\sqrt{2D_{\textrm{JS}}(Q\parallel P)}$, respectively.
Thus, if two distributions are close under those $f$-divergence measures, so are they under
total variation distance. For this reason, our theoretical analysis is based on the total
variation distance.

\paraspace
\paragraph{$\delta$-cover and $(\delta,\beta)$-cover.}
We introduce two new notions for analyzing our algorithm.
The first is the notion of \emph{$\delta$-cover}.
Given a data distribution $P$ over $\mathcal{X}$ and a value $\delta\in(0,1]$, 
if a generator $G$ satisfies $g(x)\ge\delta\cdot p(x)$ at a data point $x\in\mathcal{X}$,
we say that $x$ is $\delta$-covered by $G$ under distribution $P$.
Using this notion, the pointwise mode coverage~\eq{pointwise} states that $x$ is $\psi$-covered
by $G$ under distribution $P$ for all $x\in\mathcal{X}$.
We also extend this notion to a measurable subset
$\mathcal{S}\in\mathcal{B}(\mathcal{X})$:
we say that $\mathcal{S}$ is $\delta$-covered by $G$ under distribution $P$
if $G(\mathcal{S})\ge\delta\cdot P(\mathcal{S})$ is satisfied.

Next, consider another distribution $Q$ over $\mathcal{X}$.
We say that $G$ can \emph{$(\delta,\beta)$-cover} $(P,Q)$, if 
the following condition holds:
\begin{equation}\label{eq:cover_def}
\Pr_{x\sim Q}[x \text{ is $\delta$-covered by } G \text{ under distribution} P]\geq \beta.
\end{equation}
For instance, using this notation, Equation~\eqref{eq:yao_example} in 
our game-theoretic analysis states that $G$ can $(0.25,0.4)$-cover $(P,Q)$.

\vspace{-0.04in}
\subsection{Guarantee of Pointwise Data Coverage}\label{sec:coverage}
\vspace{-0.04in}
In each iteration of \algref{mwu_training}, we expect the generator $G_t$ to
approximate the given data distribution $P_t$ sufficiently well. We now formalize this
expectation and understand its implication. Our intuition is that by finding
a property similar to \eq{yao_example}, we should be able to establish a pointwise coverage
lower bound in a way similar to our analysis in \secref{game}. 
Such a property is given by the following lemma (and proved in \appref{proof_lem1}).
\vspace{-0.5mm}
\begin{restatable}{lemma}{lemtotalvar}\label{lem:continuous_total_variation_distance}
    Consider two distributions, $P$ and $Q$, over the data space $\mathcal{X}$,
    and a generator $G$ producing samples in $\mathcal{X}$.
For any $\delta,\gamma\in (0,1]$, if $D_{TV}(G\parallel Q)\leq \gamma,$ then $G$
can $(\delta,1-2\delta-\gamma)$-cover $(P,Q)$.
\end{restatable}
\vspace{-1mm}
Intuitively, when $G$ and $Q$ are identified, $\gamma$ is set.
If $\delta$ is reduced, then more data points in $\mathcal{X}$
can be $\delta$-covered by $G$ under $P$. Thus, the probability defined
in~\eq{cover_def} becomes larger, as reflected by the increasing $1-2\delta-\gamma$.
On the other hand, consider a fixed $\delta$. As the discrepancy between $G$ and $Q$ 
becomes larger, $\gamma$ increases. Then, sampling an $x$ according to  
$Q$ will have a smaller chance to land at a point that is $\delta$-covered by $G$
under $P$, as reflected by the decreasing $1-2\delta-\gamma$.

Next, we consider \algref{mwu_training} and identify a sufficient condition
under which the output mixture of generators $\bm{G}^*$ covers every data point
with a lower-bounded guarantee (i.e., our goal~\eq{pointwise}).
Simply speaking, this sufficient condition is as follows:
in each round $t$, the generator $G_t$ is trained such that
given an $x$ drawn from distribution $P_t$, the probability of $x$ being 
$\delta$-covered by $G_t$ under $P$ is also lower bounded.
A formal statement is given in the next lemma (proved in Appendix~\ref{sec:proof_lem2}).
\vspace{-0.5mm}
\begin{restatable}{lemma}{lemcontinuouscoverageguarantee}\label{lem:continuous_coverage_guarantee}
Recall that $T\in\mathbb{N}_{>0}$ and $\delta\in(0,1)$ are the input parameters of \algref{mwu_training}.
For any $\varepsilon \in [0,1)$ and any measurable subset $\mathcal{S}\in\mathcal{B}(\mathcal{X})$
whose probability measure satisfies $P(\mathcal{S})\ge 1/2^{\eta T}$ with some $\eta\in(0,1)$,
if in every round $t\in[T]$, $G_t$ can $(\delta,1-\varepsilon)$-cover $(P,P_t)$,
then the resulting mixture of generators $\bm{G}^*$ can
$(1-{\varepsilon}/{\ln 2}-\eta)\delta$-cover $\mathcal{S}$ under distribution $P$.
\end{restatable} 
\vspace{-1mm}
This lemma is about lower-bounded coverage of a measurable subset $\mathcal{S}$,
not a point $x\in\mathcal{X}$. At first sight, it is not of the exact form 
in~\eq{pointwise} (i.e., pointwise $\delta$-coverage).
This is because formally speaking it makes no sense to talk about covering probability
at a single point (whose measure is zero). But as $T$ approaches to $\infty$,
$\mathcal{S}$ that satisfies $P(\mathcal{S})\ge 1/2^{\eta T}$ can also approach
to a point (and $\eta$ approaches to zero). Thus, \lemref{continuous_coverage_guarantee} provides a condition for pointwise
lower-bounded coverage in the limiting sense. In practice, the provided dataset 
is always discrete, and the probability measure at each discrete data point
is finite. Then, \lemref{continuous_coverage_guarantee} is indeed
a sufficient condition for pointwise lower-bounded coverage. 

%

From \lemref{continuous_total_variation_distance}, 
we see that the condition posed by 
\lemref{continuous_coverage_guarantee} is indeed satisfied by
our algorithm, and combing both lemmas yields our final theorem (proved in Appendix~\ref{sec:proof_main}).
\vspace{-0.5mm}
\begin{restatable}{theorem}{maintheorem}\label{thm:main}
Recall that $T\in\mathbb{N}_{>0}$ and $\delta\in(0,1)$ are the input parameters of \algref{mwu_training}.
For any measurable subset $\mathcal{S}\in\mathcal{B}(\mathcal{X})$
whose probability measure satisfies $P(\mathcal{S})\ge 1/2^{\eta T}$ with some $\eta\in(0,1)$,
if in every round $t\in[T]$, $D_{\textrm{TV}}(G_t\parallel P_t)\le\gamma$,
then the resulting mixture of generators $\bm{G}^*$ can
$(1-(\gamma+2\delta)/\ln 2-\eta)\delta$-cover $\mathcal{S}$ under distribution $P$.
\end{restatable}
\vspace{-1.5mm}
In practice, existing generative models (such as GANs) can approximate $P_t$ sufficiently well,
and thus $D_{\textrm{TV}}(G_t\parallel P_t)\le\gamma$ is always satisfied for some $\gamma$.
According to \thmref{main}, a pointwise lower-bounded coverage can be obtained by 
our \algref{mwu_training}. If we choose to use a more expressive generative model (e.g., a GAN with a stronger network architecture), 
then $G_t$ can better fit $P_t$ in each round, yielding a smaller $\gamma$ used in \thmref{main}. 
Consequently, the pointwise lower bound of the data coverage becomes larger,
and effectively the coefficient $\psi$ in~\eq{pointwise} becomes larger.





\vspace{-0.04in}
\subsection{Insights from the Analysis}\label{sec:insights}
\vspace{-0.04in}
\paragraph{$\gamma$, $\eta$, $\delta$, and $T$ in \thmref{main}.}
In \thmref{main}, $\gamma$ 
depends on the expressive power of the generators being used.
It is therefore determined once the generator class $\mathcal{G}$ is chosen.
But $\eta$ can be directly set by the user and a smaller $\eta$ demands a larger $T$
to ensure $P(\mathcal{S})\ge 1/2^{\eta T}$ is satisfied.
Once $\gamma$ and $\eta$ is determined, we can choose the best $\delta$ by maximizing
the coverage bound (i.e., $(1-(\gamma+2\delta)/\ln 2-\eta)\delta$) in \thmref{main}. For example, 
if $\gamma\leq 0.1,\eta \leq 0.01$, then $\delta\approx 1/4$ would optimize the coverage bound
(see \appref{choice_delta} for more details), and in this case the coefficient $\psi$ in \eq{pointwise}
is at least $1/30$.

\vspace{-0.3mm}
\thmref{main} also sets the tone for the training cost. As explained in \appref{choice_delta}, 
given a training dataset of size $n$, the size of the generator mixture, $T$,
needs to be at most $O(\log n)$. This theoretical bound is consistent with our experimental results
presented in \secref{experiment}. In practice, only a small number of generators are needed.




\paraspace
\vspace{-0.3mm}
\paragraph{Estimated density function $g_t$.} 
The analysis in \secref{coverage} assumes that the generated probability density $g_t$ of
the generator $G_t$ in each round is known, while in practice we have to estimate $g_t$
by training a discriminator $D_t$ (recall Section~\ref{sec:alg}).
Fortunately, only mild assumptions in terms of the quality of $D_t$ 
are needed to retain the pointwise lower-bounded coverage. 
Roughly speaking, $D_t$ needs to meet two conditions:
1) In each round $t$, only a fraction of the covered data points (i.e., those 
with $g_t(x)\ge \delta\cdot p(x)$) is falsely classified by $D_t$ and doubled 
their weights.
2) In each round $t$, if the weight of a data point $x$ is not doubled based on 
the estimation of $D_t(x)$, then there is a good chance that $x$ is truly 
covered by $G_t$ (i.e., $g_t(x)\ge \delta\cdot p(x)$).
A detailed and formal discussion is presented in \appref{edf}.
In short, our estimation of $g_t$ would not deteriorate the efficacy of the algorithm,
as also confirmed in our experiments.

\paraspace
\vspace{-0.5mm}
\paragraph{Generalization.}
An intriguing question for all generative models is their \emph{generalization}
performance: {how well can a generator trained on an empirical distribution 
(with a finite number of data samples) generate samples that follow the true data distribution?}
While the generalization performance has been long studied for supervised classification,
generalization of generative models remains a widely open theoretical question.
We propose a notion of generalization for our method, and provide a preliminary theoretical analysis.
All the details are presented in \appref{generalization}.

%


 

\vspace{-0.05in}
\section{Experiments}\label{sec:experiment}
\vspace{-0.05in}
We now present our major experimental results, while referring to \appref{exp_discussion}
for network details and more results.
We show that our mixture of generators is able to cover all the modes in various synthetic and real datasets,
while {existing methods always have some modes missed}.

\vspace{-0.05in}
Previous works on generative models used
the Inception Score~\cite{salimans2016improved} or the Fr\'{e}chet Inception
Distance~\cite{heusel2017gans} as their evaluation metric. 
But we do not use them, because they are both global measures, not
reflecting mode coverage in local regions~\cite{barratt2018note}. 
Moreover, these metrics are designed to measure the quality of generated images, which is orthogonal to our goal. For
example, one can always use a more expressive GAN in each iteration of our algorithm to obtain 
better image quality and thus better inception scores.

\begin{figure*}[t]
    \centering
    \vspace{-1mm}
    \includegraphics[width=0.98\textwidth]{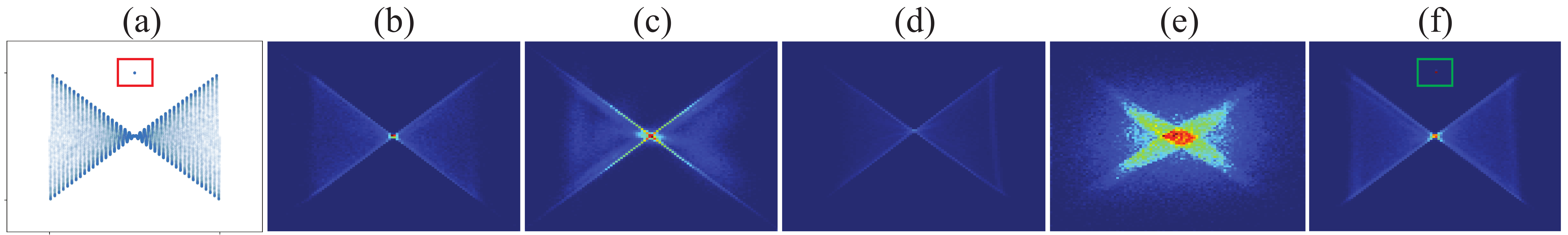}
    \vspace{-2mm}
        \caption{ \textbf{Generative models on synthetic dataset.} 
        (a) The dataset consists of two modes: one major mode as an expanding sine curve ($y=x\sin\frac{4x}{\pi}$)
        and a minor mode as a Gaussian located at $(10,0)$ (highlighted in the reb box).
        (b-f) We show color-coded distributions of generated samples from (b) EM, (c) GAN, (d) AdaGAN, (e) VAE,
        and (f) our method (i.e., a mixture of GANs). Only our method is able to cover the second mode 
        (highlighted in the green box; zoomin to view).
        \label{fig:sin_comp}}
    \vspace{-5mm}
\end{figure*}

\vspace{-0.05in}
Since the phenomenon of missing modes is particularly prominent in GANs, our
experiments emphasize on the mode coverage performance of GANs and compare our
method (using a mixture of GANs) with DCGAN~\cite{radford2015unsupervised},
MGAN~\cite{hoang2018mgan}, and AdaGAN. The latter two also use multiple GANs to
improve mode coverage, although they do \emph{not} aim for the same mode coverage notion as ours.


\paraspace
\vspace{-0.5mm}
\paragraph{Overview.} 
We first outline all our experiments, including those presented in \appref{exp_discussion}. 
\textbf{i)} We compare our method with a number of classic generative models on a synthetic dataset. 
\textbf{ii)} In \appref{exp_discussion_adagan}, we also compare our
method with AdaGAN~\cite{tolstikhin2017adagan} on other synthetic datasets as well as stacked MNIST dataset, 
because both are boosting algorithms aiming at improving mode coverage.
\textbf{iii)} We further compare our method with a single large DCGAN, 
AdaGAN, and MGAN on the 
Fashion-MNIST dataset~\cite{xiao2017fashion} mixed with a very small portion 
of MNIST dataset~\cite{lecun1998gradient}. 

\paraspace
\paragraph{Various generative models on synthetic dataset.} 
As we show in \appref{other_types}, many generative models, such as expectation-maximization (EM) methods,
VAEs, and GANs, all rely on a global statistical distance in their training.
We therefore test their mode coverage and compare with ours.
We construct on $\mathbb{R}^2$ a synthetic dataset with two modes.
The first mode consists of data points whose $x$-coordinate is uniformly sampled by $x_i\sim [-10, 10]$ 
and the $y$-coordinate is $y_i=x_i\sin\frac{4x_i}{\pi}$. The second mode has data points forming a Gaussian 
at $(0,10)$.
The total number of data points in the first mode is 400$\times$ of the second.
As shown in \figref{sin_comp}, generative models include EM, GAN, VAE,
and AdaGAN~\cite{tolstikhin2017adagan} all fail to cover the second mode.
Our method, in contrast, captures both modes.
We run KDE to estimate the likelihood of our generated samples on our synthetic data experiments (using KDE
bandwidth=0.1). We compute $L = 1/N \sum_{i} P_{model}(x_i)$, where $x_i$ is a sample in the minor mode. For the minor mode,
our method has a mean log likelihood of -1.28, while AdaGAN has only -967.64 (almost no samples from AdaGAN).



\begin{wraptable}[13]{r}{0.52\textwidth}
    \vspace{-6mm}
    \begin{center}
    \begin{small}
        \begin{tabular}{lccc}
            \toprule
            & ``1''s  & Frequency  & Avg Prob. \\ \hline
            DCGAN   &  $13$      & $0.14\times 10^{-4}$ & $0.49$ \\ 
            MGAN    &  collapsed    & -       & - \\ 
            AdaGAN  &  $60$     & $0.67\times 10^{-4}$ & $0.45$ \\ 
            Our method &  $289$ & $3.2\times 10^{-4}$ &$0.68$\\ 
            \bottomrule
        \end{tabular}
        \vspace{-3.5mm}
        \caption{ \textbf{Ratios of generated images classified as ``1''.} We
        generate $9\times 10^5$ images from each method. The second column indicates the
        numbers of samples being classified as ``1'', and the third column indicates
        the ratio. In the fourth column, we average the prediction probabilities
        over all generated images that are classified as ``1''.}
        \label{tab:freq}
        \end{small}
    \end{center}   
\end{wraptable}
\paraspace
\vspace{-0.5mm}
\paragraph{Fashion-MNIST and partial MNIST.} 
Our next experiment is to challenge different GAN models
with a real dataset that has separated and unbalanced modes.
This dataset consists of the entire training dataset of Fashion-MNIST (with 60k images) mixed with
randomly sampled $100$ MNIST images labeled as ``1''. 
The size of generator mixture is always set to be $30$ for
AdaGAN, MGAN and our method, and all generators share the same network
structure. Additionally, 
when comparing with a \emph{single} DCGAN,
we ensure that the DCGAN's total number of parameters is comparable
to the total number of parameters of the $30$ generators in AdaGAN, MGAN, and
ours.

\vspace{-0.02in}
To evaluate the results, we train an 11-class classifier to distinguish the 10 classes in
Fashion-MNIST and one class in MNIST (i.e., ``1''). First, we check how many samples from each 
method are classified as ``1''. The test setup and results are shown in~\tabref{freq} 
and its caption.  The results suggest that our method can
generate more ``1'' samples with higher prediction confidence. 
Note that MGAN has a strong mode collapse and fails to produce ``1'' samples. 
While DCGAN and AdaGAN generate some samples that are classified as ``1'', inspecting
the generated images reveals that those samples are all visually far from ``1''s, but incorrectly 
classified by the pre-trained classifier (see \figref{topconfidence}). In contrast, our 
method is able to generate samples close to ``1''. We also note that our method can produce
higher-quality images if the underlying generative models in each round become stronger.

 \begin{figure*}[t]
    \centering
    \vspace{-1mm}
    \includegraphics[width=0.99\textwidth]{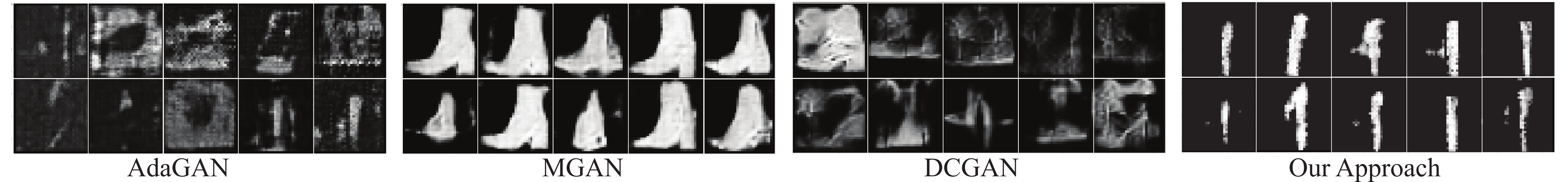}
    \vspace{-3mm}
        \caption{ \textbf{Most confident ``1'' samples.} 
        Here we show samples that are generated by each tested methods
        and also classified by the pre-trained classifier most confidently 
        as ``1'' images (i.e., top 10 in terms of the classified probability).
        Samples of our method are visually much closer to ``1''.
        \label{fig:topconfidence}}
    \vspace{-3mm}
\end{figure*}




\begin{wrapfigure}[12]{r}{0.468\textwidth}
    \centering
    \vspace{-1mm}
    \includegraphics[width=0.468\textwidth]{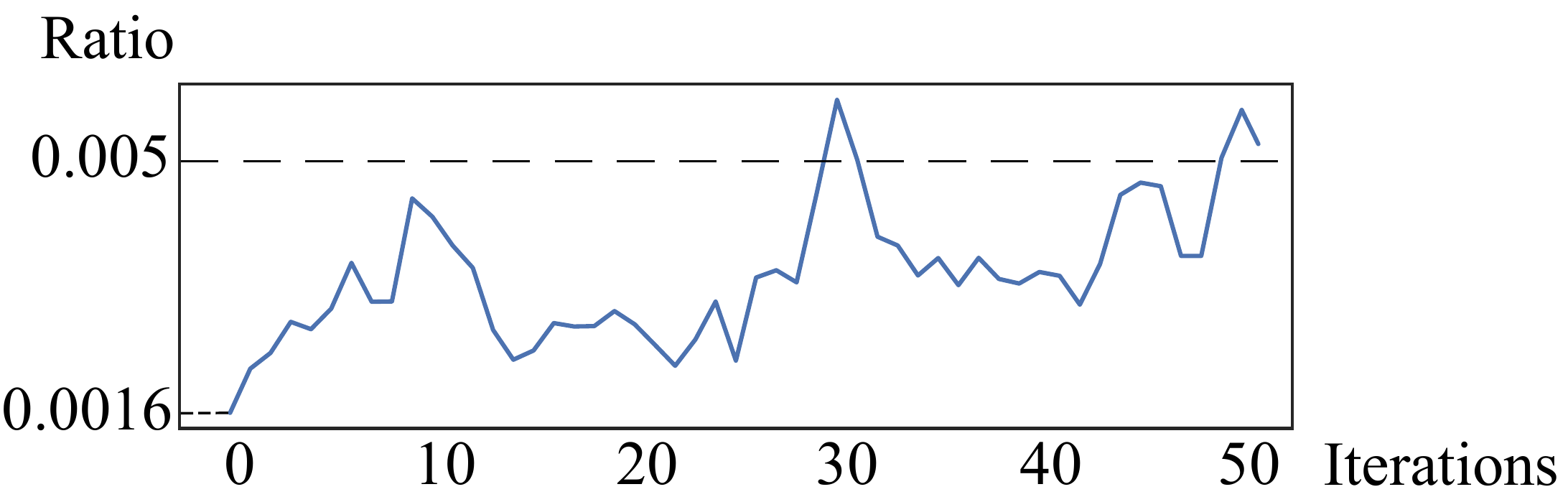}
    \vspace{-6mm}
    \caption{ \textbf{Weight ratio of ``1''s.} We calculate the ratio of the total weights of
        training images labeled by ``1'' to the total weights of all training images in each
                round, and plot here how the ratio changes with respect to the iterations
                in our algorithm.
    }\label{fig:totalweightsratio}
\end{wrapfigure}
\textbf{Another remarkable feature} is observed in our algorithm. In each round of our training algorithm,
we calculate the total weight $\bar{w}_t$ of provided training samples classified as ``1''
as well as the total weight $W_t$ of all training samples. When plotting the ratio 
$\bar{w}_t/W_t$ changing with respect to the number of rounds (\figref{totalweightsratio}),
interestingly, we found that this ratio has a maximum value at around 0.005 
in this example. We conjecture that in the training dataset if the ratio of 
``1'' images among all training images is around $1/200$, then 
a single generator may learn and generate ``1'' images (the minority mode).
To verify this conjecture, we trained a GAN (with the same network structure)
on another training dataset with 60k training
images from Fashion-MNIST mixed with $300$ MNIST ``1'' images.
We then use the trained generator to sample 100k images. As a result,
In a fraction of $4.2\times 10^{-4}$, those images are classified as ``1''.  
Figure~\ref{fig:single_DC_highconf} in \appref{exp_discussion} shows some of those images.
This result confirms our conjecture and suggests that $\bar{w}_t/W_t$ may be used as a measure of \emph{mode bias}
in a dataset.

\begin{figure*}[t]
    \centering
    \includegraphics[width=0.99\columnwidth]{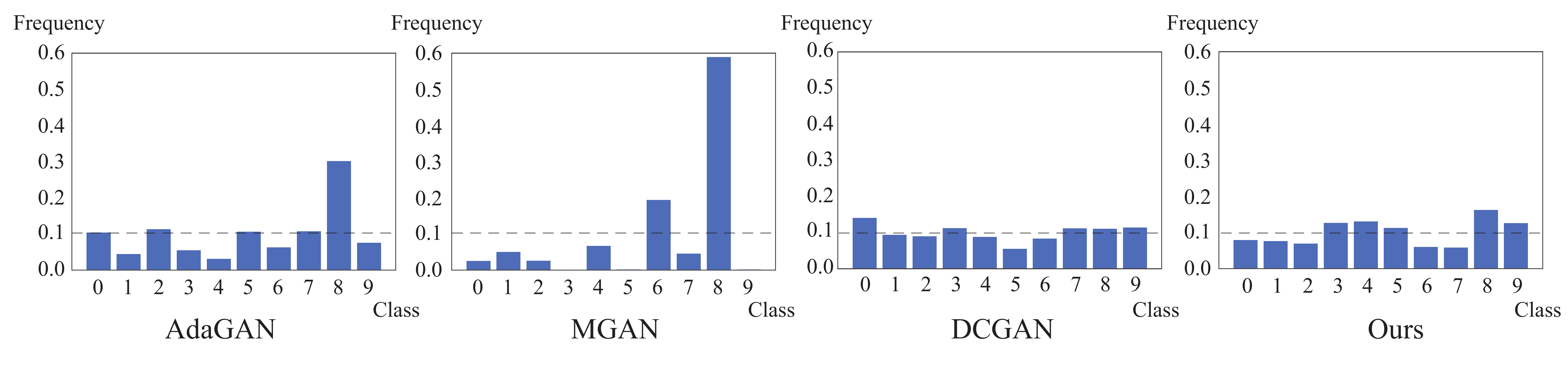}
    \vspace{-4mm}
        \caption{ \textbf{Distribution of generated samples.} 
        Training samples are drawn uniformly from each class.
        But generated samples by AdaGAN and MGAN are considerably nonuniform,
        while those from DCGAN and our method are more uniform.
        This experiment suggests that the conventional heuristic of reducing a statistical distance
        might not merit its use in training generative models.
    }\label{fig:unbalanced_samples_for_fashion}
    \vspace{-4mm}
\end{figure*}


Lastly, in \figref{unbalanced_samples_for_fashion}, 
we show the generated distribution over the 10 Fashion-MNIST classes from each tested method.
We neglect the class ``1'', as MGAN fails to generate them.
The generated samples of AdaGAN and MGAN is highly nonuniform, though in the training
dataset, the 10 classes of images are uniformly distributed.
Our method and DCGAN produce more uniform samples.
This suggests that although other generative models (such as AdaGAN and MGAN)
aim to reduce a global statistical distance, the generated samples may not easily match
the empirical distribution---in this case, a uniform distribution. 
Our method, while not aiming for reducing the statistical
distance in the first place, matches the target empirical distribution plausibly, as a byproduct.





\vspace{-0.05in}
\section{Conclusion}
\vspace{-0.05in}
We have presented an algorithm that iteratively trains a mixture of generators, driven 
by an explicit notion of complete mode coverage. 
With this notion for designing generative models, 
our work poses an alternative goal, 
one that differs from the conventional training philosophy:
instead of reducing a global statistical distance between the target distribution and generated distribution, 
one only needs to make the distance mildly small but not have to reduce it toward a perfect zero, 
and our method is able to boost the generative model with theoretically guaranteed mode coverage.

\paraspace
\paragraph{Acknowledgments.}
This work was supported in part by the National Science Foundation (
CAREER-1453101, 1816041, 1910839, 1703925, 1421161, 1714818, 1617955, 1740833),  Simons Foundation
(\#491119 to Alexandr Andoni), Google Research Award,
a Google PhD Fellowship,
a Snap Research Fellowship, 
a Columbia SEAS CKGSB Fellowship, 
and SoftBank Group.

\newpage
\bibliographystyle{unsrt}
\bibliography{ref}

\newpage
\appendix
\begin{center}
\Large
\textbf{Supplementary Document}\\ 
\smallskip
\textbf{Rethinking Generative Mode Coverage:\\ A Pointwise Guaranteed Approach}
\medskip
\end{center}

\section{Global Statistic Distance Based Generative Approaches}\label{sec:other_types}
In this section, we analyze a few classic generative models to show their
connections to the reduction of a certain global statistical distance.
The reliance on global statistical distances
explains why they suffer from missing modes, as empirically confirmed in \figref{sin_comp} of the main text.

\paragraph{Maximum Likelihood Estimation.}
Consider a target distribution $P$ with density function $p(\cdot)$. 
Suppose we are provided with $n$ i.i.d. samples $\{x_1,x_2,\cdots,x_n\}$ drawn from $P$.
The goal of training a generator through maximum likelihood estimation (MLE) is to find 
from a predefined generator family $\mathcal{G}$ 
the generator $G$ that maximize 
$$ L(G) = \frac{1}{n}\sum_i \log g(x_i),$$
where $g(\cdot)$ is the probability density function of the distribution generated by $G$. 
When $n$ approaches $\infty$, the MLE objective amount to 
\begin{align*}
    \lim_{n\to\infty} \left(\max_{G\in\mathcal{G}}L(G)\right) = \max_{G\in\mathcal{G}}\E_{x \sim P}[\log g(x)]=\max_{G\in \mathcal{G}}\int p(x)\log g(x) \mathrm{d}x
    =\min_{G\in \mathcal{G}}\left(-\int p(x)\log g(x) \mathrm{d}x\right),
\end{align*} 
which is further equivalent to solve the following optimization problem:
\begin{align*}
\int p(x)\log p(x) \mathrm{d}x+\min_{G\in\mathcal{G}}\left(-\int p(x)\log g(x) \mathrm{d}x\right) = \min_{G\in\mathcal{G}} D_{\textrm{KL}}(P\parallel G).
\end{align*}
This is because the first term on the LHS is irrelevant from $G$ and thus is a constant.
From this expression, it is evident that the goal of MLE is to minimize a global statistical distance, namely, KL-divergence.

Figure~\ref{fig:mle} illustrates an 1D example wherein the MLE fails to achieve pointwise coverage.
Although \figref{mle}, for pedagogical purpose, involves a generator family $\mathcal{G}$ consisting of only two generators,
it is by no means a pathological case, since
in practice generators always have limited expressive power, 
limited by a number of factors.
For GANs, it is limited by the structure of generators.
For VAEs, it is the structure of encoders and decoders.
For Gaussian Mixture models, it is the 
dimension of the space and the number of mixture components.  
Given a $\mathcal{G}$ with limited expressive power, MLE cannot guarantee complete mode coverage.

%

\begin{figure}[t!]
    \centering
    \includegraphics[width=0.6\textwidth]{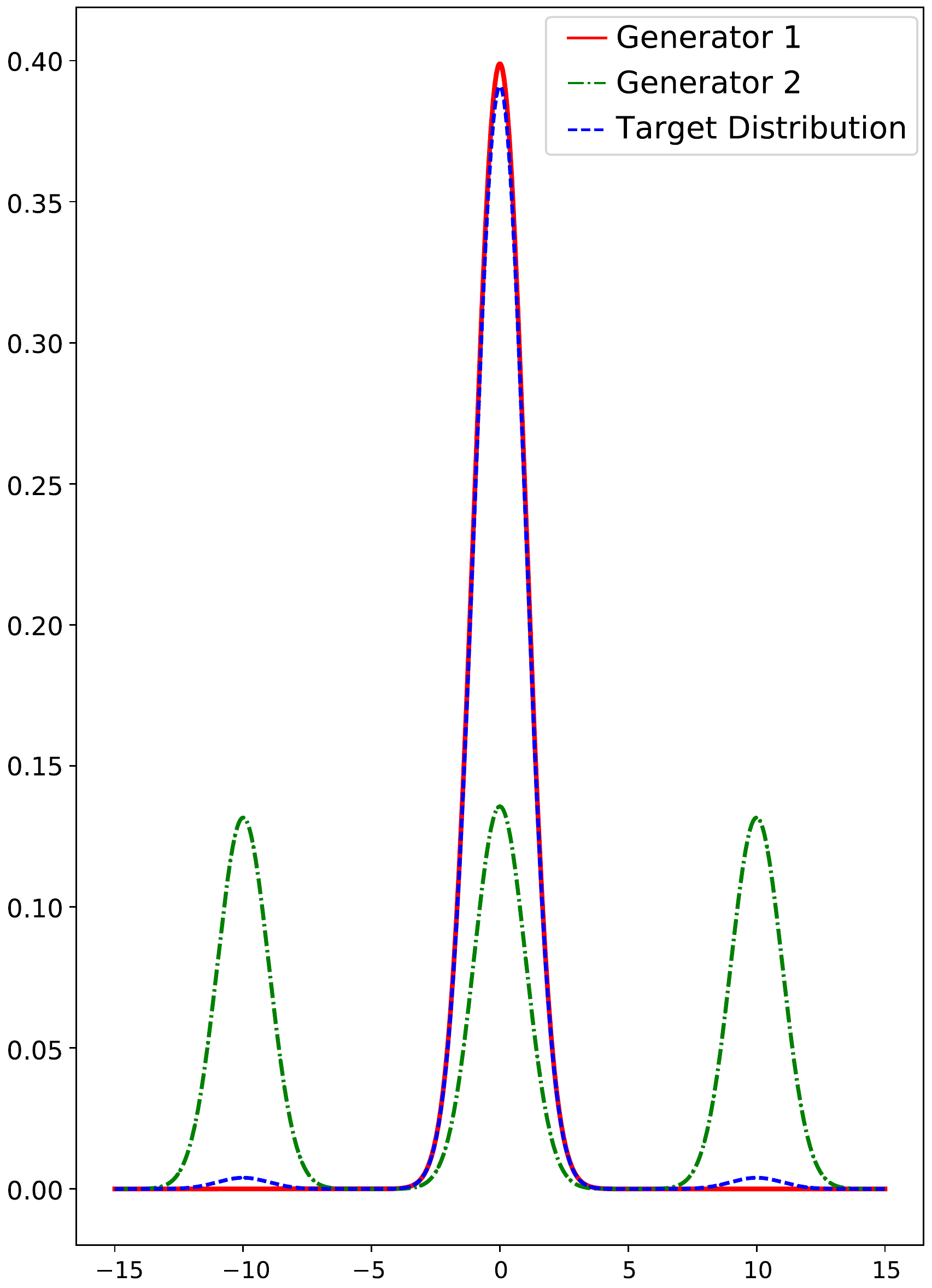}
    \vspace{-3mm}
    \caption{ 
        Consider a 1D target distribution $P$ with three modes, i.e., a mixture of three Gaussians,
        $P = 0.98\cdot\mathcal{N}(0,1)+0.01\cdot\mathcal{N}(10,1)+0.01\cdot\mathcal{N}(-10,1)$.
        In this example, the generator class $\mathcal{G}$ only contains two generators.
        The generated distribution of the first generator $G_1$ is $\mathcal{N}(0,1)$,
        while the distribution of the second generator $G_2$ is $0.34\cdot\mathcal{N}(0,1)+0.33\cdot\mathcal{N}(10,1)+0.33\cdot\mathcal{N}(-10,1)$. 
        In this case, we have
        $D_{\mathrm{KL}}(P,G_1)\approx1.28$, $D_{\mathrm{KL}}(P,G_2)\approx
        1.40$, $D_{\mathrm{KL}}(G_1,P)\approx
        0.029$, and $D_{\mathrm{KL}}(G_2,P)\approx2.81$ (all $D_{\mathrm{KL}}$ measures
        use a log base of $2$).
        To minimize $D_{\mathrm{KL}}(P,G)$,
        maximum likelihood estimation method will choose the first generator, $G_1$.
        The probability of drawing samples from the side modes (in $[-14,-6]$ and $[6,14]$) of the target distribution 
        $P$ is $\Pr_{x\sim P}[6\leq |x|\leq 14]\approx 0.02$, but
        the probability of generating samples from the first generator in the same intervals is
        $\Pr_{x\sim G_1}[6\leq |x|\leq 14]\approx 10^{-9}$.
        Thus, the side modes are almost missed. 
        To make the first generator satisfy Equation~\eqref{eq:pointwise}, we have to choose $\psi\approx 10^{-7}$,
        which in practice implies no pointwise coverage guarantee.
        In contrast, the generated distribution of the second generator can
        satisfy Equation~\eqref{eq:pointwise} with $\psi>1/3$, which is a
        plausible pointwise coverage guarantee.
    } \label{fig:mle}
    \vspace{-3mm}
\end{figure}

\paragraph{Variational Autoencoders (VAEs).} 
A VAE has a encoder $\theta\in \Theta$ and a decoder $\phi\in \Phi$
chosen from an encoder and decoder families, $\Theta$ and $\Phi$. It
also needs a known prior distribution $Q$ (whose probability density is $q(\cdot)$) of latent variable $z$.  
Provided a decoder $\phi$ and the prior distribution $Q$, we can
construct a generator $G$: to generate an $x$, we firstly sample a latent
variable $z\sim Z$ and then sample an $x$ according to the (approximated)
likelihood function $p_{\phi}(x|z)$.  
To train a VAE, a target distribution $P$ is provided and the training objective is
\begin{align}
\max_{\theta\in\Theta,\phi\in \Phi}\int_{x} p(x)\cdot \mathrm{ELBO}_{\theta,\phi}(x) \mathrm{d}x,
\end{align}
where $\mathrm{ELBO}_{\theta,\phi}(x)$ is called the evidence lower bound, defined as
\begin{align}
\mathrm{ELBO}_{\theta,\phi}(x)=\int_{z} p_{\theta}(z|x) \log p_{\phi}(x|z)  \mathrm{d}z - \int_{z} p_{\theta}(z|x) \log\left(\frac{p_{\theta}(z|x)}{q(z)}\right)  \mathrm{d}z,
\end{align}
Here $p_{\theta}(z|x)$ is the (approximated) posterior function.

Let $G\in\mathcal{G}$ be a generator corresponding to the decoder $\phi$ and the prior $Z$,
and let $g(\cdot)$ be the generative probability density of $G$.
Then, we have the following derivation:
\begin{equation}
\begin{split}
&\E_{x\sim P}[\log g(x)]=\int_{x} p(x) \log g(x) \mathrm{d}x=\int_{x} p(x) \int_{z} p_{\theta}(z|x) \log\left(g(x)\right)\mathrm{d}z \mathrm{d}x\\
=&\int_{x} p(x) \int_{z} p_{\theta}(z|x) \log\left(\frac{p_{\phi}(x|z)q(z)}{p_{\phi}(z|x)}\right)\mathrm{d}z \mathrm{d}x\\
=&\int_{x} p(x) \int_{z} p_{\theta}(z|x) \log\left(\frac{p_{\phi}(x|z)q(z)p_{\theta}(z|x)}{p_{\phi}(z|x)p_{\theta}(z|x)}\right)\mathrm{d}z \mathrm{d}x\\
=&\int_{x} p(x) \left(\int_{z} p_{\theta}(z|x) \log\left(\frac{p_{\phi}(x|z)q(z)}{p_{\theta}(z|x)}\right)\mathrm{d}z + \int_{z} p_{\theta}(z|x) \log\left(\frac{p_{\theta}(z|x)}{p_{\phi}(z|x)}\right)\mathrm{d}z\right) \mathrm{d}x\\
=&\int_{x} p(x) \left(\int_{z} p_{\theta}(z|x) \log\left(\frac{p_{\phi}(x|z)q(z)}{p_{\theta}(z|x)}\right)\mathrm{d}z + D_{\mathrm{KL}}\left(p_{\theta}(z|x)\parallel p_{\phi}(z|x)\right)\right) \mathrm{d}x\\
=&\int_{x} p(x) \left(
\mathrm{ELBO}_{\theta,\phi}(x)
+ D_{\mathrm{KL}}\left(p_{\theta}(z|x)\parallel p_{\phi}(z|x)\right)\right) \mathrm{d}x.
\end{split}
\end{equation}
Notice that $D_{\mathrm{KL}}\left(p_{\theta}(z|x)\parallel p_{\phi}(z|x)\right)$ is always non-negative
and it reaches $0$ when $p_{\theta}(z|x)$ is the same as $p_{\phi}(z|x)$. This means 
\begin{align*}
\E_{x\sim P}[\log g(x)]\geq \int_{x} p(x) \cdot  \mathrm{ELBO}_{\theta,\phi}(x) \mathrm{d}x.
\end{align*}
If $\theta$ is perfectly trained, i.e., $p_{\theta}(z|x)$ matches exactly $p_{\phi}(z|x)$, then 
\begin{align*}
\max_{G\in\mathcal{G}} \E_{x\sim P}[\log g(x)] = \max_{\theta\in\Theta,\phi\in \Phi}\int_{x} p(x)\cdot \mathrm{ELBO}_{\theta,\phi}(x) \mathrm{d}x.
\end{align*}
From this perspective, it becomes evident that optimizing a VAE essentially amounts to a maximum likelihood estimation.
Depending on the generator family $\mathcal{G}$ (determined by $\Phi$ and $Z$) and the encoder family $\Theta$, 
mode collapse may not always happen.
But since it is essentially a maximum likelihood estimation method, the
pointwise mode coverage~\eqref{eq:pointwise} can not be guaranteed in
theory, as discussed in the previous paragraph.

\paragraph{Generative Adversarial Networks (GANs).} 
Given a target distribution $P$, the objective of training a GAN~\cite{goodfellow2014generative} is to solve the following
optimization problem:
\begin{align*}
\min_{G\in\mathcal{G}}\max_{D} L(G,D),
\end{align*}
where $L(G,D)$ is defined as
\begin{align*}
L(G,D)=\E_{x\sim P}[\log(D(x))]+\E_{x\sim G}[\log(1-D(x))]=\int_{x} p(x)\log(D(x)) + g(x)\log(1-D(x))\mathrm{d}x.
\end{align*}
As shown in \cite{goodfellow2014generative}, the optimal discriminator $D^*$ of Nash equilibrium satisfies $D^*(x)\equiv 1/2$.
When using $D^*$ in $L(G,D)$, we have 
\begin{align*}
L(G,D^*)= D_{\mathrm{KL}}\left(P \parallel\frac{P+G}{2}\right)+D_{\mathrm{KL}}\left(G \parallel\frac{P+G}{2}\right) -2 = 2 D_{\mathrm{JS}}(P \parallel G) - 2,
\end{align*}
where $D_{\mathrm{JS}}$ is the Jensen-Shannon divergence.
Thus, GAN essentially is trying to reduce the global statistical distance, measured by Jensen-Shannon divergence.

There are many variants of GANs, which use (more or less) different loss functions $L(G,D)$ in training. 
But all of them still focus on reducing a global statistical distance.
For example, the loss function of the Wasserstein GAN \cite{arjovsky2017wasserstein} is $\E_{x\sim P}[D(x)]-\E_{x\sim G}[D(x)]$. 
Optimizing such a loss function over all $1$-Lipschitz $D$ is essentially to reduce the Wasserstein distance, 
another global statistical distance measure.

\begin{algorithm}[t]
    \caption{Training on empirical distribution}
    \label{alg:empirical}
    \begin{algorithmic}[1]
        \STATE {\bfseries Parameters:} $T$, a positive integer number of generators, 
        and $\delta\in(0,1)$, a covering threshold. 
        \STATE {\bfseries Input:} a set $\{x_i\}_{i=1}^n$ of i.i.d. samples drawn from an unknown data distribution $P$.
        \STATE For each $x_i$, initialize its weight $w_1(x_i)=1/n$.  
        \FOR{$t = 1\rightarrow T$}
        \STATE Construct an empirical distribution $\widehat{P}_t$ such that each $x_i$ is drawn with probability $\frac{w_t(x_i)}{W_t}$, where $W_t=\sum_i w_t(x_i)$.
        \STATE Train $G_t$ on i.i.d. samples drawn from $\widehat{P}_t$.
        \STATE Train a discriminator $D_t$ to distinguish the samples from $\widehat{P}_t$ and the samples from $G_t$.
        \STATE For each $x_i$, if 
        $\left(\frac{1}{D_t(x_i)}-1\right)\cdot \frac{w_t(x_i)}{W_t}<\frac{\delta}{n}$,
        set $w_{t+1}(x_i) = 2 \cdot w_t(x_i)$. \\
        Otherwise, set $w_{t+1}(x_i)=w_t(x_i)$.
        \ENDFOR
        \STATE {\bfseries Output:} a mixture of generators $\bm{G}^*=\{G_1,\ldots,G_T\}$.
    \end{algorithmic}
\end{algorithm}

\section{Algorithm on Empirical Dataset}\label{sec:alg_emp}
In practice, the provided dataset $\{x_i\}_{i=1}^n$
consists of $n$ i.i.d.~samples from $P$. 
According to the Glivenko-Cantelli theorem \cite{cantelli1933sulla}, 
the uniform distribution over $n$ i.i.d.~samples from $P$
will converge to $P$ when $n$ approaches to infinity. 
As a simple example,
let $P$ be a discrete distribution over two points, $A$ and $B$, with $P(A)=\nicefrac{5}{7}$ and
$P(B)=\nicefrac{2}{7}$. If $7$ samples are drawn from $P$ to form the input data, ideally they
should be a multiset $\{A,A,A,A,A,B,B\}$. Each sample has a weight $\nicefrac{1}{7}$, and the total
weights of $A$ and $B$ are $\nicefrac{5}{7}$ and $\nicefrac{2}{7}$. 
Then we will train a generator $G_1$ from the training distribution where point $A$ has training probability $\nicefrac{5}{7}$ and point $B$ has training probability $\nicefrac{2}{7}$.

If the generator $G_1$ obtained is collapsed, e.g., $G_1$ samples $A$ with probability $1$ and samples $B$ with probability $0$, 
then ideally the discriminator $D_1$ will satisfy $D_1(A)=\nicefrac{5}{12}$ and $D_1(B)=1$.
Suppose the parameter $\delta=1/4$ in \algref{mwu_training} (and \algref{empirical}). We have 
\begin{align*}
\left(\frac{1}{D_1(A)}-1\right)\cdot \frac{w_1(A)}{W_1(A)} =
\left(\frac{1}{D_1(A)}-1\right)\cdot \frac{5}{7}\cdot \frac{1}{5}\geq \delta \cdot P(A)\cdot \frac{1}{5} = \delta/n =\frac{\nicefrac{1}{4}}{7}
\end{align*}
and
\begin{align*}
\left(\frac{1}{D_1(B)}-1\right)\cdot \frac{w_1(B)}{W_1(B)} = \left(\frac{1}{D_1(B)}-1\right)\cdot \frac{2}{7}\cdot \frac{1}{2}< \delta\cdot P(B)\cdot \frac{1}{2} = \delta/n =\frac{\nicefrac{1}{4}}{7}.
\end{align*}
Thus, each sample $B$ will double the weight, and each sample $A$ will remain the same weight unchanged. 
The total weight of $A$ is $\nicefrac{5}{7}$, and the total weight of $B$ is $\nicefrac{4}{7}$.
In the second iteration, the total probability of $A$ will be decreased to $\nicefrac{5}{9}$ and the total probability of $B$ will be increased to $\nicefrac{4}{9}$. 
We will use the new probability to train the generator $G_2$ and the discriminator $D_2$, and repeat the above procedure.

In practice, we do not need to know the probability density $p(x)$ of $P$;
every sample $x_i$ is considered to have a finite and uniform probability measure.
After the generator $G$ is trained over this dataset, 
its generated sample distribution should approximate well the data distribution $P$.
In light of this, the \algref{mwu_training} can be implemented empirically as what is outlined
in \algref{empirical}.

\section{Statistical Distance from Lower-bounded Pointwise Coverage }\label{sec:local_compatible_global}
Equation~\eqref{eq:pointwise} (i.e., $\forall x\in\mathcal{X}, g(x)\geq \psi \cdot p(x)$)
is a pointwise lower-bounded data coverage that we pursue in this paper.
If Equation~\eqref{eq:pointwise} is satisfied, then 
the total variation distance between $P$ and $G$ is automatically upper bounded, because
\begin{align*}
    D_{\text{TV}}(P\parallel Q) = &\frac{1}{2}\int_{\mathcal{X}}|p(x)-g(x)|\mathrm{d}x
=\int_{\mathcal{X}}\mathds{1}{(p(x)>g(x))}\cdot (p(x)-g(x))\mathrm{d}x\\
\leq~&\int_{\mathcal{X}}\mathds{1}{(p(x)>g(x))}\cdot (p(x)-\psi\cdot p(x))\mathrm{d}x\\
=~& (1-\psi)\cdot \int_{\mathcal{X}}\mathds{1}{(p(x)>g(x))}\cdot p(x) \mathrm{d}x\\
\leq~& 1-\psi.
\end{align*}


\section{Proof of Equation~\eqref{eq:yao_example}}\label{sec:proof_game}
Suppose two arbitrary distributions $P$ and $Q$ are defined over a data space $\mathcal{X}$. 
$G$ is the distribution of generated samples over $\mathcal{X}$. 
If the total variation distance between $Q$ and $G$ is at most $0.1$, then we have
\begin{align*}
\Pr_{x\sim Q}\left[g(x)\geq \frac{1}{4}p(x)\right]
=~&\int_{\mathcal{X}} \mathds{1}{\left(g(x)\geq \frac{1}{4}p(x)\right)} \cdot q(x) \mathrm{d}x \\
\geq~&\int_{\mathcal{X}} \mathds{1}{\left(g(x),q(x)\geq \frac{1}{4}p(x)\right)} \cdot q(x) \mathrm{d}x\\
=~&\int_{\mathcal{X}} \mathds{1}{\left(q(x)\geq \frac{1}{4}p(x)\right)} \cdot q(x) \mathrm{d}x - \int_{\mathcal{X}} \mathds{1}{\left(q(x)\geq \frac{1}{4}p(x)> g(x)\right)} \cdot q(x) \mathrm{d}x\\
\geq~ & \frac{3}{4} - \int_{\mathcal{X}} \mathds{1}{\left(q(x)\geq \frac{1}{4}p(x)> g(x)\right)} (q(x) - g(x) + g(x)) \mathrm{d}x\\
\geq~ & \frac{3}{4} - 0.1  - \frac{1}{4} = 0.4,
\end{align*}
where the first term of the right-hand side of the second inequality follows from
\begin{align*}
\int_{\mathcal{X}} \mathds{1}{\left(q(x)\geq \frac{1}{4}p(x)\right)} \cdot q(x) \mathrm{d}x = 1 - \int_{\mathcal{X}} \mathds{1}{\left(q(x)< \frac{1}{4}p(x)\right)} \cdot q(x) \mathrm{d}x
\geq 1- \int_{\mathcal{X}} \frac{1}{4} p(x)\mathrm{d}x=\frac{3}{4}.\end{align*}
And the third inequality follows from
\begin{align*}
\int_{\mathcal{X}} \mathds{1}{\left(q(x)\geq \frac{1}{4}p(x)> g(x)\right)} (q(x) - g(x)) \mathrm{d}x\leq \int_{\mathcal{X}} \mathds{1}{\left(q(x)> g(x)\right)} (q(x) - g(x)) \mathrm{d}x \leq 0.1,
\end{align*}
and
\begin{align*}
\int_{\mathcal{X}} \mathds{1}{\left(q(x)>\frac{1}{4}p(x)> g(x)\right)}  g(x) \mathrm{d}x \leq \int_{\mathcal{X}} \mathds{1}{\left(\frac{1}{4}p(x)> g(x)\right)}  g(x) \mathrm{d}x \leq \int_{\mathcal{X}}\frac{1}{4}p(x)\mathrm{d}x\leq \frac{1}{4}.
\end{align*}

\section{Theoretical Analysis Details}
In this section, we provide proofs of the lemmas and theorem presented in Section~\ref{sec:analysis}.
We repeat the statements of the lemmas and theorem before individual proofs.
We also provide details to further elaborate the discussion provided in \secref{insights} 
of the paper.

We follow the notations introduced in \secref{analysis} of the main text.
In addition, we will use $\log(\cdot)$ to denote $\log_2(\cdot)$ for short.

\subsection{Proof of Lemma~\ref{lem:continuous_total_variation_distance}}\label{sec:proof_lem1}
\lemtotalvar*
\begin{proof}
Since $D_{\textrm{TV}}(G||Q)\leq \gamma$ and
$\int_{\mathcal{X}}q(x)\mathrm{d}x=\int_{\mathcal{X}}g(x)\mathrm{d}x=1$, we know that
\begin{equation}\label{eq:int_dtv}
D_{\text{TV}}(G\parallel Q) = \frac{1}{2}\int_{\mathcal{X}}|q(x)-g(x)|\mathrm{d}x=
\int_{\mathcal{X}} \mathds{1}{(q(x)>g(x))}\cdot (q(x)-g(x))\mathrm{d}x\leq \gamma. 
\end{equation}
Next, we derive a lower bound of $\Pr_{x\sim Q}[x \text{ is $\delta$-covered by } G
\text{ under }P]$:
\begin{align*}
&\Pr_{x\sim Q}[x \text{ is $\delta$-covered by } G \text{ under }P]\\
=~& \int_{\mathcal{X}} \mathds{1}{(g(x)\geq \delta\cdot p(x))} \cdot q(x)\mathrm{d}x
\geq \int_{\mathcal{X}} \mathds{1}{(g(x),q(x)\geq \delta\cdot p(x))} \cdot q(x) \mathrm{d}x\\
=~&\int_{\mathcal{X}} \mathds{1}{(q(x)\geq \delta\cdot p(x))} \cdot q(x)\mathrm{d}x - \int_{\mathcal{X}}\mathds{1}{(q(x)\geq \delta\cdot p(x)>g(x))} \cdot q(x)\mathrm{d}x\\
=~& 1 - \int_{\mathcal{X}} \mathds{1}{(q(x)< \delta\cdot p(x))} \cdot q(x)\mathrm{d}x - \int_{\mathcal{X}}\mathds{1}{(q(x)\geq \delta\cdot p(x)>g(x))} \cdot q(x)\mathrm{d}x\\
\geq~& 1 - \delta\int_{\mathcal{X}}p(x)\mathrm{d}x  - \int_{\mathcal{X}}\mathds{1}{(q(x)\geq \delta\cdot p(x)>g(x))} \cdot q(x)\mathrm{d}x\\
=~& 1 - \delta - \int_{\mathcal{X}}\mathds{1}{(q(x)\geq \delta\cdot p(x)>g(x))} \cdot (q(x) - g(x) + g(x))\mathrm{d}x\\
=~& 1 - \delta - \int_{\mathcal{X}}\mathds{1}{(q(x)\geq \delta\cdot p(x)>g(x))} \cdot (q(x) - g(x))\mathrm{d}x - \int_{\mathcal{X}}\mathds{1}{(q(x)\geq \delta\cdot p(x)>g(x))} \cdot g(x) \mathrm{d}x\\
\geq~& 1 -\delta- \gamma - \int_{\mathcal{X}}\mathds{1}{(q(x)\geq \delta\cdot p(x)>g(x))} \cdot g(x) \mathrm{d}x\\
\geq~& 1- \delta- \gamma - \delta\int_{\mathcal{X}} p(x) \mathrm{d} x= 1-2\delta-\gamma,
\end{align*} 
where the first equality follows from definition, the second equality follows from $\mathds{1}(q(x)\geq \delta\cdot p(x))=\mathds{1}(g(x),q(x)\geq \delta\cdot p(x))+\mathds{1}(q(x)\geq \delta\cdot p(x)>g(x))$,  
the third inequality follows from Equation~\eqref{eq:int_dtv}, and the last inequality follows from
\begin{align*}
\int_{\mathcal{X}} \mathds{1}(q(x)\geq \delta\cdot p(x)>g(x))\cdot g(x)\mathrm{d}x\leq \int_{\mathcal{X}}\mathds{1}(\delta\cdot p(x)>g(x))\cdot g(x)\mathrm{d}x \leq \int_{\mathcal{X}}\delta \cdot p(x) \mathrm{d}x.
\end{align*} 
\end{proof}

\subsection{Proof of Lemma~\ref{lem:continuous_coverage_guarantee}}\label{sec:proof_lem2}
Here we first assume that the probability density $g_t$ of generated samples
is known.
In Appendix~\ref{sec:edf}, we will consider the case where
$g_t$ is estimated by a discriminator as described in Section~\ref{sec:alg}.

\lemcontinuouscoverageguarantee*

\begin{proof}
First, we consider the total weight $W_{t+1}$ after $t$ rounds, we derive the following
upper bound:
\begin{align*}
W_{t+1} & =\int_{\mathcal{X}} w_{t+1}(x) \mathrm{d}x
 = \int_{\mathcal{X}} w_t(x)\cdot (1+ 
 \mathds{1}{(g_t(x)<\delta\cdot p(x))})\mathrm{d}x\\
& = W_t + 
 W_t\cdot \int_{\mathcal{X}} \mathds{1}{(g_t(x)<\delta\cdot p(x))}\cdot \frac{w_t(x)}{W_t}\mathrm{d}x\\
& = W_t + 
W_t\cdot \int_{\mathcal{X}} \mathds{1}{(g_t(x)<\delta\cdot p(x))}\cdot p_t(x)\mathrm{d}x\\
& = W_t + 
W_t\cdot \Pr_{x\sim P_t}[g_t(x)<\delta\cdot p(x)]\\
& = W_t + W_t\cdot (1-\Pr_{x\sim P_t}[g_t(x)\geq\delta\cdot p(x)])\\
& \leq W_t + W_t\cdot (1-(1-\varepsilon))\\
& \leq W_t\cdot (1+
\varepsilon),
\end{align*}
where the first equality follows from definition, the second equality follows from Line 9
of Algorithm~\ref{alg:mwu_training}, the forth equality follows from the construction of distribution $P_t$. 
In addition, the first inequality follows from that $G_t$ can $(\delta,1-\varepsilon)$-cover $(P,P_t).$
Thus, $W_{T+1}\leq W_1\cdot (1+\varepsilon)^T=(1+\varepsilon)^T.$ 

On the other hand, we have
\begin{equation}\label{eq:aa1}
\begin{split}
W_{T+1} & =\int_{\mathcal{X}} w_{T+1}(x) \mathrm{d}x
\geq \int_{\mathcal{S}} w_{T+1}(x) \mathrm{d}x
\geq \int_{\mathcal{S}} 
2^{\sum_{t=1}^T \mathds{1}{(g_t(x)<\delta \cdot p(x))}} p(x)\mathrm{d}x\\
&=\E_{x\sim P}\left[
2^{\sum_{t=1}^T \mathds{1}{(g_t(x)<\delta \cdot p(x))}} \middle|x\in \mathcal{S}\right] \Pr_{x\sim P}[x\in \mathcal{S}],
\end{split}
\end{equation}
where the first equality follows from definition, the first inequality follows from
$\mathcal{S}\subseteq\mathcal{X}$, and the second inequality follows from Line 9 of
Algorithm~\ref{alg:mwu_training}.
Dividing both sides by $\Pr_{x\sim P}[x\in\mathcal{S}]$ of~\eq{aa1} and taking the logarithm yield
\begin{equation} \label{eq:non_cover_exp}
\begin{split}
    \log\left(\frac{W_{T+1}}{\Pr_{x\sim P}[x\in\mathcal{S}]}\right) 
 \geq ~&
 \log\left(\E_{x\sim P}\left[
 2^{\sum_{t=1}^T \mathds{1}{(g_t(x)<\delta \cdot p(x))}} \middle|x\in
 \mathcal{S}\right]\right) \\ 
 \geq ~& \E_{x\sim P}\left[\sum_{t=1}^T \mathds{1}{(g_t(x)<\delta \cdot p(x))} \middle|x\in \mathcal{S}\right],
\end{split}
\end{equation}
where the last inequality follows from Jensen's inequality.

Lastly, we have a lower bound for $\Pr_{x\sim G}[x\in\mathcal{S}]$:
\begin{align*}
\Pr_{x\sim G}[x\in\mathcal{S}]&
=\int_{\mathcal{S}} \frac{1}{T}\sum_{t=1}^T g_t(x) \mathrm{d}x
\geq\int_{\mathcal{S}} \frac{1}{T} \sum_{t=1}^T (\mathds{1}{(g_t(x)\geq \delta\cdot p(x))} \cdot g_t(x))\mathrm{d}x\\
&\geq \int_{\mathcal{S}} \frac{1}{T} \sum_{t=1}^T (\mathds{1}{(g_t(x)\geq \delta\cdot p(x))}\cdot \delta\cdot p(x))\mathrm{d}x\\
&= \frac{\delta}{T} \int_{\mathcal{S}} \sum_{t=1}^T  \mathds{1}{(g_t(x)\geq \delta\cdot p(x))}  \cdot p(x)\mathrm{d}x\\
&= \frac{\delta}{T} \E_{x\sim P}\left[\sum_{t=1}^T \mathds{1}{(g_t(x)\geq \delta\cdot p(x))}\middle|x\in \mathcal{S}\right]\cdot \Pr_{x\sim P}[x \in \mathcal{S}]\\
&= \frac{\delta}{T}  \left(T-\E_{x\sim P}\left[\sum_{t=1}^T \mathds{1}{(g_t(x)<\delta \cdot p(x))} \middle|x\in \mathcal{S}\right]\right)\cdot \Pr_{x\sim P}[x \in \mathcal{S}]\\
&\geq \delta(1-\log
(W_{T+1}/\Pr_{x\sim P}[x\in \mathcal{S}])/T)\cdot\Pr_{x\sim P}[x\in \mathcal{S}]\\
&\geq\delta(1-\varepsilon/\ln 2 - \eta)\cdot \Pr_{x\sim P}[x\in \mathcal{S}],
\end{align*}
where the third inequality follows from Equation~\eqref{eq:non_cover_exp}, while the last
inequality follows from $\log(W_{T+1})\leq \log((1+\varepsilon)^T)\leq \varepsilon T/\ln 2$
and $\Pr_{x\sim P}[x\in \mathcal{S}]=P(\mathcal{S})\geq 1/2^{\eta T}$.
\end{proof}

\subsection{Proof of Theorem~\ref{thm:main}}\label{sec:proof_main}
\maintheorem*

\begin{proof}
From Lemma~\ref{lem:continuous_total_variation_distance}, we have $\forall t\in [T]$,
$G_t$ can $(\delta,1-\gamma-2\delta)$-cover $(P,P_t)$. Combining it with 
Lemma~\ref{lem:continuous_coverage_guarantee}, we have $\forall \mathcal{S}\subseteq \mathcal{X}$
with $P(\mathcal{S})\geq 1/2^{\eta T},$ $G$ can $(1-(\gamma+2\delta)/\ln
2-\eta)\delta$-cover $\mathcal{S}$.
\end{proof}

\subsection{Choice of $T$ and $\delta$ according to \thmref{main}}\label{sec:choice_delta}
Suppose the empirical dataset has $n$ data points independently sampled from a target distribution $P$.
We claim that in our train algorithm, $T=O(\log n)$ suffices.
This is because if a subset $\mathcal{S}\in\mathcal{B}(\mathcal{X})$ has a sufficiently 
small probability measure, for example, $P(\mathcal{S})<1/n^3$, then with a 
high probability (i.e., at least $1-1/n^2$), no data samples in $\{x_i\}_{i=1}^n$
is located in $\mathcal{S}$. In other words, the empirical dataset of size $n$ reveals
almost no information of a subset $\mathcal{S}$ if $P(\mathcal{S})<1/n^3$,
or equivalently if $1/2^{\eta T}\approx 1/n^3$ (according to \thmref{main}).
This shows that $T=O(\log n)$ suffices.

\thmref{main} also sheds some light on the choice of $\delta$ in \algref{mwu_training} 
(and \algref{empirical} in practice).
We now present the analysis details for choosing $\delta$.
We use $\mathcal{G}$ to denote the type of generative models trained in 
each round of our algorithm.
According to Theorem~\ref{thm:main}, if we know $\eta$ (depends on $T$) and $\gamma$
(depends on $\mathcal{G}$), then we wish to maximize the lower bound 
$(1-(\gamma+2\delta)/\ln 2-\eta)\delta$ over $\delta$, and the optimal $\delta$ is $\frac{(1-\eta)\ln 2 -\gamma}{4}$.  
Although in practice $\gamma$ is unknown and not easy to estimate,  
we note that $\gamma$ is relatively small in practice, and $\eta$ can be also small when we increase
the number of rounds $T$. 

Given two arbitrary distributions $P$ and $Q$ over $\mathcal{X}$, 
if the total variation distance between $Q$ and a generated distribution $G$ is at most
$\gamma$ (as we discussed in \secref{game} of the main text), then we have
\begin{align*}
\Pr_{x\sim Q}[g(x)\geq \delta \cdot p(x)]
=~&\int_{\mathcal{X}} \mathds{1}(g(x)\geq \delta p(x))\cdot q(x)\mathrm{d}x\\
\geq~& \int_{\mathcal{X}} \mathds{1}(g(x),q(x)\geq \delta\cdot p(x))\cdot q(x)\mathrm{d}x\\
=~& \int_{\mathcal{X}} \mathds{1}(q(x)\geq \delta\cdot p(x))\cdot q(x)\mathrm{d}x - \int_{\mathcal{X}}\mathds{1}(q(x)\geq \delta\cdot p(x)>g(x))\cdot q(x)\mathrm{d}x\\
\geq~& 1-\delta - \int_{\mathcal{X}} \mathds{1}(q(x)
\geq \delta \cdot p(x)> g(x)) (q(x)-g(x)+g(x))\mathrm{d}x\\
\geq~& 1-\delta-\gamma -\delta = 1-2\delta-\gamma.
\end{align*}
As discussed in Section~\ref{sec:game}, we can find a mixture of
generators satisfying pointwise $(1-2\delta-\gamma)\delta$-coverage.  Letting
$\gamma = 0$, we see that the optimal choice of $\delta$ in this setting is $1/4$. 
And in this case, $(1-2\delta)\delta=1/8$ is a theoretical bound of the coverage ratio by our algorithm.

\subsection{Use of Estimated Probability Density $g_t$}\label{sec:edf}
In Algorithm~\ref{alg:mwu_training}, we use a discriminator $D_t$ to estimate 
the probability density $g_t$ of generated samples of each generator $G_t$.
The discriminator $D_t$ might not be perfectly trained, causing inaccuracy of estimating $g_t$.
We show that the pointwise lower-bound in our data coverage is retained if 
two mild conditions are fulfilled by $D_t$.
\begin{enumerate}[itemsep=0.5pt,leftmargin=11pt]
    \item In each round, only a bounded fraction of covered data points
$x$ (i.e., those with $g_t(x)\ge\delta\cdot p(x)$) is falsely classified
and their weights are unnecessarily doubled.
Concretely, $\forall t\in[T],$ if a sample $x$ is drawn from
distribution $P_t$, then the probability of both events---$x$ is $\delta$-covered by
$G_t$ under $P$ and $\left(\frac{1}{D_t(x)}-1\right)\cdot
\frac{w_1(x)}{p(x)W_t}<\delta$---happening is bounded by $\varepsilon'$.  
\item For any data point
$x\in\mathcal{X}$, if in round $t$, the weight of $x$ is not doubled, then with a good
chance, $x$ is really $\delta'$-covered, where $\delta'$ can be smaller than $\delta$.
Formally, $\forall x\in\mathcal{X}, \left|\{t\in[T]\middle| g_t(x)\geq \delta'\cdot
p(x)\}\right|\geq
\lambda\cdot \left|\left\{t\in[T]\middle| \left(\frac{1}{D_t(x)}-1\right)\cdot
\frac{w_t(x)}{p(x)W_t}\geq \delta\right\}\right|$. Because
$\left(\frac{1}{D_t(x)}-1\right)\cdot \frac{w_t(x)}{p(x)W_t}<\delta$ happens if and only
if $w_{t+1}(x)= 2\cdot w_t(x)$, we use the event $w_{t+1}(x)=2\cdot w_t(x)$ as an
indicator of the event $\left(\frac{1}{D_t(x)}-1\right)\cdot \frac{w_t(x)}{p(x)W_t} < \delta$.
\end{enumerate}

If the condition (1) is satisfied, then we are able to upper bound the total weight
$W_{T+1}.$ Similarly to the proof of Lemma~\ref{lem:continuous_coverage_guarantee},
this can be seen from the following derivation:
\begin{align*}
W_{t+1} & =\int_{\mathcal{X}} w_{t+1}(x) \mathrm{d}x\\
& \leq \int_{\mathcal{X}} w_t(x)\cdot (1+ 
\mathds{1}{(g_t(x)<\delta\cdot p(x))}+\mathds{1}(g_t(x)\geq \delta\cdot p(x)\wedge w_{t+1}(x)= 2w_t(x)))\mathrm{d}x\\
& = W_t + 
W_t\cdot \int_{\mathcal{X}} (\mathds{1}{(g_t(x)<\delta\cdot p(x))}+\mathds{1}(g_t(x)\geq \delta\cdot p(x)\wedge w_{t+1}(x)= 2w_t(x)))\cdot \frac{w_t(x)}{W_t}\mathrm{d}x\\
& = W_t + 
W_t\cdot \int_{\mathcal{X}} (\mathds{1}{(g_t(x)<\delta\cdot p(x))}+\mathds{1}(g_t(x)\geq \delta\cdot p(x)\wedge w_{t+1}(x)= 2w_t(x)))\cdot p_t(x)\mathrm{d}x\\
& = W_t + 
W_t\cdot \Pr_{x\sim P_t}[g_t(x)<\delta\cdot p(x)]+W_t\cdot\Pr_{x\sim P_t}[g_t(x)\geq \delta\cdot p(x)\wedge w_{t+1}(x)= 2w_t(x)]\\
& \leq W_t + W_t\cdot (1-\Pr_{x\sim P_t}[g_t(x)\geq\delta\cdot p(x)]) +W_t\cdot \varepsilon'\\
& \leq W_t + W_t\cdot (1-(1-\varepsilon))+W_t\cdot \varepsilon'\\
& \leq W_t\cdot (1+
\varepsilon+\varepsilon'),
\end{align*}
Thus, the total weight $W_{T+1}$ is bounded by $(1+\varepsilon+\varepsilon')^T.$
Again in parallel to the proof of Lemma~\ref{lem:continuous_coverage_guarantee},
we have 
\begin{align*}
W_{T+1} & =\int_{\mathcal{X}} w_{T+1}(x) \mathrm{d}x
\geq \int_{\mathcal{S}} w_{T+1}(x) \mathrm{d}x
\geq \int_{\mathcal{S}} 
2^{\sum_{t=1}^T \mathds{1}{(w_{t+1}(x)=2\cdot w_t(x))}} p(x)\mathrm{d}x\\
&=\E_{x\sim P}\left[
2^{\sum_{t=1}^T \mathds{1}{(w_{t+1}(x)=2\cdot w_t(x))}} \middle| x\in \mathcal{S}\right] \Pr_{x\sim P}[x\in \mathcal{S}].
\end{align*}
Dividing both sides by $\Pr_{x\sim P}[x\in\mathcal{S}]$ yields
\begin{align*}
\log\left(\frac{W_{T+1}}{\Pr_{x\sim P}[x\in\mathcal{S}]}\right) 
& \geq 
\log\left(\E_{x\sim P}\left[
2^{\sum_{t=1}^T \mathds{1}{(w_{t+1}(x)=2w_t(x))}} \middle| x\in \mathcal{S}\right]\right) \notag\\
& \geq \E_{x\sim P}
\left[\sum_{t=1}^T \mathds{1}{(w_{t+1}(x)=2w_t(x))} \middle| x\in \mathcal{S}\right].
\end{align*}
Meanwhile, if the condition (2) is satisfied, then
\begin{multline}
\lambda\cdot\left(T-\E_{x\sim P}\left[\sum_{t=1}^T \mathds{1}{(w_{t+1}(x)=2w_t(x))}
\middle| x\in \mathcal{S}\right]\right)\leq \\ T - \E_{x\sim P}\left[\sum_{t=1}^T
\mathds{1}{(g_t(x)<\delta' \cdot p(x))} \middle|x\in \mathcal{S}\right]. \label{eq:different_part}
\end{multline}
Following the proof of Lemma~\ref{lem:continuous_coverage_guarantee},
we obtain 
\begin{align*}
\Pr_{x\sim G}[x\in\mathcal{S}]&
=\int_{\mathcal{S}} \frac{1}{T}\sum_{t=1}^T g_t(x) \mathrm{d}x
\geq\int_{\mathcal{S}} \frac{1}{T} \sum_{t=1}^T (\mathds{1}{(g_t(x)\geq \delta'\cdot p(x))} \cdot g_t(x))\mathrm{d}x\\
&\geq \int_{\mathcal{S}} \frac{1}{T} \sum_{t=1}^T (\mathds{1}{(g_t(x)\geq \delta'\cdot p(x))}\cdot \delta'\cdot p(x))\mathrm{d}x\\
&= \frac{\delta'}{T} \int_{\mathcal{S}} \sum_{t=1}^T  \mathds{1}{(g_t(x)\geq \delta'\cdot p(x))}  \cdot p(x)\mathrm{d}x\\
&= \frac{\delta'}{T} \E_{x\sim P}\left[\sum_{t=1}^T \mathds{1}{(g_t(x)\geq \delta'\cdot p(x))}\middle|x\in \mathcal{S}\right]\cdot \Pr_{x\sim P}[x \in \mathcal{S}]\\
&= \frac{\delta'}{T}  \left(T-\E_{x\sim P}\left[\sum_{t=1}^T \mathds{1}{(g_t(x)<\delta' \cdot p(x))} \middle|x\in \mathcal{S}\right]\right)\cdot \Pr_{x\sim P}[x \in \mathcal{S}]\\
&\geq \frac{\delta'\lambda}{T} \left(T-\E_{x\sim P}\left[\sum_{t=1}^T \mathds{1}{(w_{t+1}(x)=2w_t(x))} \middle|x\in \mathcal{S}\right]\right)\\
&\geq \delta'\lambda(1-\log
(W_{T+1}/\Pr_{x\sim P}[x\in \mathcal{S}])/T)\cdot\Pr_{x\sim P}[x\in \mathcal{S}]\\
&\geq\delta'\lambda(1-(\varepsilon+\varepsilon')/\ln 2 - \eta)\cdot \Pr_{x\sim P}[x\in \mathcal{S}],
\end{align*}
where the third inequality follows from Equation~\eqref{eq:different_part}, and other steps are similar to the proof in Lemma~\ref{lem:continuous_coverage_guarantee}.
By combining with Lemma~\ref{lem:continuous_total_variation_distance}, the final coverage ratio of Theorem~\ref{thm:main} with imperfect discriminators $D_t$ should be $(1-(\gamma+2\delta+\varepsilon')/\ln 2-\eta)\delta'\lambda$.

\subsection{Discussion on Generalization}\label{sec:generalization}
Recently, Arora et al.~\cite{arora2017generalization} proposed the \emph{neural net
distance} for measuring generalization performance of GANs.  However, their metric still
relies on a global distance measure of two distributions, not necessarily reflecting the
generalization for pointwise coverage.



While a dedicated answer of this theoretical question is beyond the scope of
this work, here we propose our notion of generalization and briefly discuss its
implication for our algorithm.
Provided a training dataset consisting of $n$ i.i.d. samples $\{x_i\}_{i=1}^n$ drawn from the
distribution $P$, we train
a mixture of generators $\bm{G}^*$. 
Our notion of generalization is defined as 
$\Pr_{x\sim P}[x\text{ is $\psi$-covered by } \bm{G}^*]$,
the probability of $x$ being $\psi$-covered by empirically trained $\bm{G}^*$ when 
$x$ is sampled from the true target distribution $P$.
A perfect generalization has a value 1 under this notion. 
We claim that 
given fixed $T$ rounds of our algorithm and a constant $\varepsilon\in(0,1)$,
if $G_t$ in each round is from a family $\mathcal{G}$ of generators
(e.g., they are all GANs with the same network architecture), and if $n$ is at
least $\Omega(\varepsilon^{-1}T\log |\mathcal{G}|)$, then we have the
generalization $\Pr_{x\sim P}[x\text{ is $\psi$-covered by } \bm{G}^*]\geq
1-\varepsilon$.
Here $|\mathcal{G}|$ is the
size of \emph{essentially} different generators in $\mathcal{G}$. Next, 
we elaborate this statement.

\paragraph{Generalization Analysis.}
Our analysis start with a definition of a \emph{family} of generators.
In each round of our algorithm, we train a generator $G_t$. We now identify a
family of generators from which $G_t$ is trained.  In general, a generator $G$
can be viewed as a pair $(f(\cdot),Z)$, where $Z$ is the latent space
distribution (or prior distribution) over the latent space $\mathcal{Z}$, and
$f(\cdot)$ is a transformation function that maps the latent space
$\mathcal{Z}$ to a target data domain $\mathcal{X}$. 
Let $z$ be a random variable of distribution $Z$. 
Then, the generated distribution (i.e., distribution of samples generated by $G$) 
is denoted by the distribution of $f(z)$.
For example, for GANs~\cite{goodfellow2014generative} and VAEs~\cite{kingma2013auto}, 
$f(\cdot)$ is a function represented by a neural network, and $Z$ is usually a standard Gaussian or mixture of Gaussians.

In light of this, we define a family $\mathcal{G}$ of generators 
represented by a pair $(\mathcal{F},Z)$, where $\mathcal{F}$ is a set of
functions mapping from $\mathcal{Z}$ to $\mathcal{X}$. 
For example, in the framework of GANs, $\mathcal{F}$ can be expressed by a
neural network with a finite number of parameters which have bounded values. 
If the input to the neural network (i.e., the latent space) is also bounded,
then we are able to apply \emph{net} argument (see e.g., \cite{arora2017generalization}) to find a finite subset $\mathcal{F}'\subset
\mathcal{F}$ such that for any $f\in\mathcal{F}$, there exists a function
$f'\in\mathcal{F'}$ sufficiently close to $f$. 
Then the size of $\mathcal{F'}$, denoted by $|\mathcal{F'}|$, can be regarded 
as the number of ``essentially different'' functions (or neural networks). 

Recall that the generator family $\mathcal{G}$ can be represented by $(\mathcal{F},Z)$. If
the latent space $Z$ is fixed (such as a single Gaussian), then we can
define ``essentially different'' generators in a way similar 
to the definition of ``essentially different'' functions in $\mathcal{F}$.
If the number of ``essentially different'' generators from
$\mathcal{G}$ is finite, we define the size of $\mathcal{G}$
as $|\mathcal{G}|$.



With this notion, the number of different mixture of generators $\bm{G}^*=\{G_1,...,G_T\}$ is at most $|\mathcal{G}|^T$.
Consider a uniform mixture $\bm{G}^*$ of generators,
$G_1,G_2,\cdots,G_T\in \mathcal{G}$. If $\Pr_{x\sim P}[x\text{ is \textit{not}
$\psi$-covered by }\bm{G}^*]\geq \varepsilon,$ then for $n$ i.i.d. samples
$x_1,x_2,\cdots,x_n\sim P$, the probability that every $x_i$ is $\psi$-covered
by $G$ is at most $(1-\varepsilon)^n$, that is, 
\begin{equation*}
\Pr_{x_1,...,x_n\sim P}\left[\text{every }x_1,...,x_n\text{ is $\psi$-convered by }\bm{G}^*\right] \le (1-\varepsilon)^n.
\end{equation*}
Next, by union bound over all possible mixtures
$\bm{G}^*$ that satisfies $\Pr_{x\sim P}[x\text{ is not $\psi$-covered by }\bm{G}^*]\geq \varepsilon$,
we have the following probability bound:
\begin{multline}\label{eq:fff}
\Pr_{x_1,...,x_n\sim P}\left[\exists\bm{G}^*\text{s.t. }\Pr_{x\sim P}\left[x \text{ is \emph{not} $\psi$-covered by }\bm{G}^*\right]\geq \varepsilon\,
\text{and every }x_1,...,x_n\text{ is $\psi$-convered by }\bm{G}^*\right] \\ \le (1-\varepsilon)^n |\mathcal{G}|^T.
\end{multline}
Thus, if $n\geq \Omega(\varepsilon^{-1}T\log|\mathcal{G}|)$, then with a high probability, 
the inverse of the probability condition above is true, because in this case $(1-\varepsilon)^n$ on the right-hand side of~\eq{fff} is 
small---that is, with a high probability,
for any mixture $\bm{G}^*$ that satisfies $\Pr_{x\sim P}[x\text{ is \emph{not} $\psi$-covered by }G]\geq
\varepsilon,$ there must exist a sample $x_i$ such that $x_i$ cannot be $\psi$-covered by $\bm{G}^*$.
The occurrence of this condition implies that 
if we find a generator mixture $\bm{G}^*$ that can $\psi$-cover every $x_i$, then
$\Pr_{x\sim P}[x\text{ is $\psi$-covered by } G]\geq 1-\varepsilon$. 
In other words, we conclude that
if we have $n\geq \Omega(\varepsilon^{-1}T\log |\mathcal{G}|)$ i.i.d. samples
$\{x_i\}_{i=1}^n$ drawn from the distribution $P$, and if our algorithm finds
a mixture $\bm{G}^*$ of generators that can $\psi$-cover every $x_i$, then with a high probability, 
our notion of generalization has $\Pr_{x\sim P}[x\text{ is $\psi$-covered by } G]\geq 1-\varepsilon$. 

\section{Experiment Details and More Results}\label{sec:exp_discussion}

\subsection{Network Architecture and Training Hyperparameters.}
In our tests, we construct a mixture of GANs.
The network architecture of the GANs in show in
\tabref{structure1} for experiments on synthetic datasets and in \tabref{structure2}
for real image datasets. 
All experiments use Adam optimizer~\cite{kingma2014adam}
with a learning rate of $10^{-3}$, and we set $\beta_1=0.5$ and $\beta_2=0.999$
with a mini-batch size of 128.

\begin{table}[h]
    \small
    \centering
    \begin{tabular}{ccccc}
    \whline{1.0pt}
         & layer 		   & output size & activation function  \\
    \whline{0.7pt}
         & input (dim 10)  &  10	     &      \\
         & Linear &           32 	     & ReLU  \\
         & Linear &           32 	     & ReLU  \\
         & Linear &           2 	     &   \\
    \whline{1.0pt}
    \end{tabular}
    \vspace{1mm}
    \caption{\textbf{Network structure} for synthetic data generator. } \label{tab:structure1}
    \vspace{-2mm}
\end{table}

\begin{table}[h]
    \small
    \centering
    \begin{tabular}{ccccccc}
    \whline{1.0pt}
         & layer 		   & output size & kernel size & stride & BN  & activation function  \\
    \whline{0.7pt}
         & input (dim 100)  &  100$\times$1$\times$1 	 &              &        &     &      \\
         & Transposed Conv &  512$\times$4$\times$4 	 & 4 		& 1 	 & Yes & ReLU  \\
         & Transposed Conv &  256$\times$8$\times$8 	 & 4  		& 2 	 & Yes & ReLU  \\
         & Transposed Conv &  128$\times$16$\times$16    & 4 		& 2 	 & Yes & ReLU  \\
         & Transposed Conv &  channel$\times$32$\times$32  & 4 		& 2 	 & No  & Tanh  \\
    \whline{1.0pt}
    \end{tabular}
    \vspace{1mm}
            \caption{\textbf{Network structure} for image generator. channel=3 for Stacked
            MNIST and channel=1 for FasionMNIST+MNIST.} \label{tab:structure2}
    \vspace{-2mm}
\end{table}

\subsection{Additional Experiment Details on Real Data}

\paragraph{Stacked MNIST dataset.} 
Stacked MNIST is an augmentation of MNIST dataset~\cite{che2016mode} for
evaluating mode collapse.  We randomly sample three images from MNIST dataset
and stack them in RGB channels of an image.  
In this way, we construct a dataset of 100k images, each of which has a dimension of
$32\times 32\times 3$. 

\paragraph{Pre-trained classifier.} 
For Fashion-MNIST with partial MNIST dataset, we use all the training data of
Fashion-MNIST and MNIST to train a 11-class classifier. For stacked MNIST dataset, we
train a 10-class classifier on MNIST, and use it as a 1000-class classifier
on stacked MNIST (by applying the 10-class MNIST classifier on each color channel). For
each experiment, we regard each class as a mode, and use the 
pre-trained classifier to classify the generated samples into individual modes.
After classifying generated samples, we can estimate the generation probability for each mode. 


\subsection{Comparison with AdaGAN on Synthetic Dataset and Stacked MNIST}\label{sec:exp_discussion_adagan}

\begin{figure}[t]
	\centering
	\includegraphics[width=0.5\textwidth]{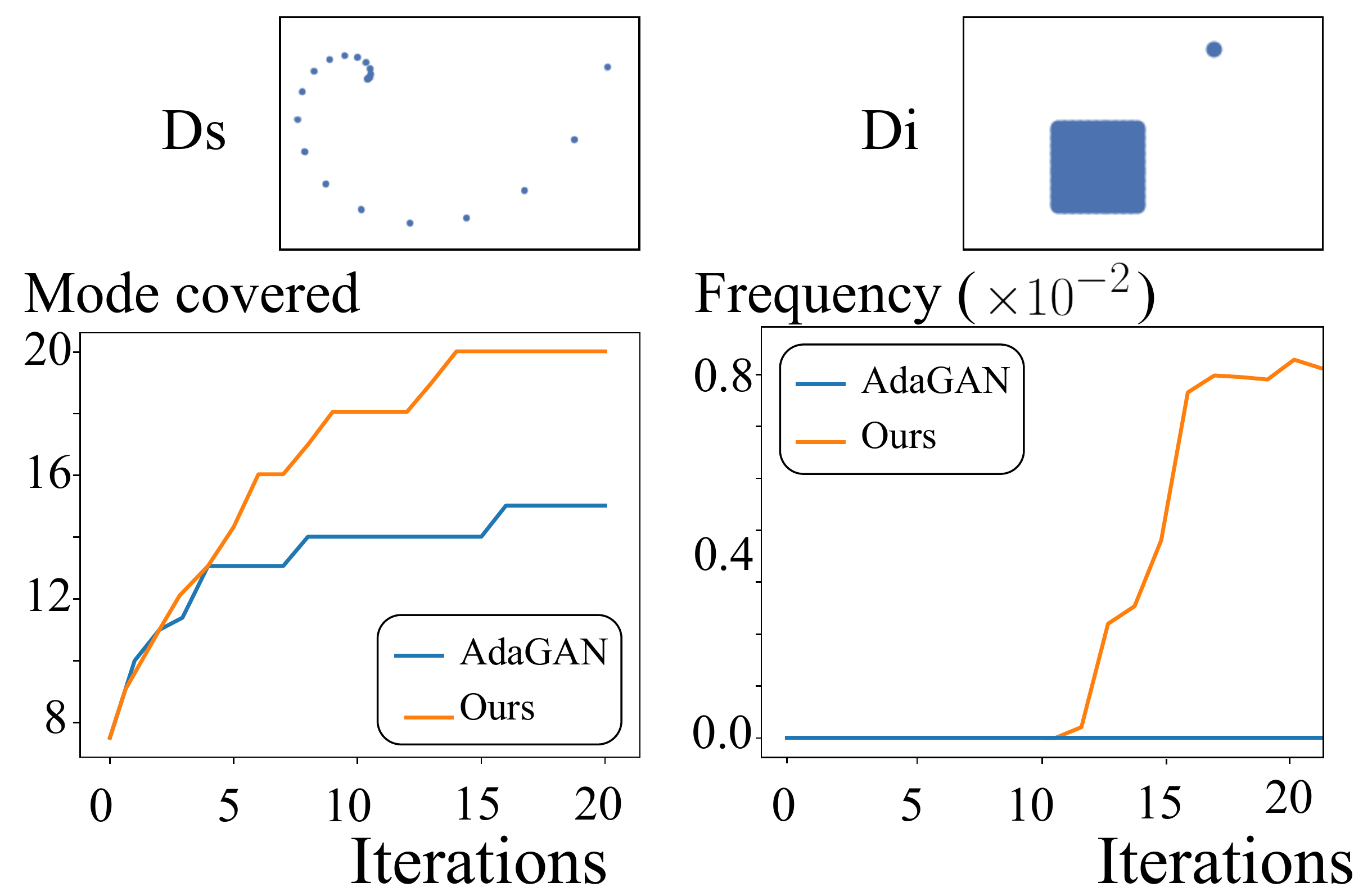}
	\vspace{-2mm}
        \caption{ \textbf{Challenging datasets.} 
	We compare our method with AdaGAN on two datasets (top). 
        (left) Our method covers all modes in $\mathsf{D}_s$ dataset with 20 iterations in average. 
        (right) Our method increases the sampled frequency (sampling weights) 
        of the separate mode as the training iteration progresses, whereas 
        AdaGAN increases the sampling frequency of the separated modes.
        Eventually, AdaGAN can only cover 14 modes in $\mathsf{D}_s$ and never
        cover the separated mode in $\mathsf{D}_i$.  In contrast, our method
        successfully covers all modes.
	}\label{fig:spiral}
\end{figure}

 \begin{figure}[t]
\centering
    \includegraphics[width=0.4\columnwidth]{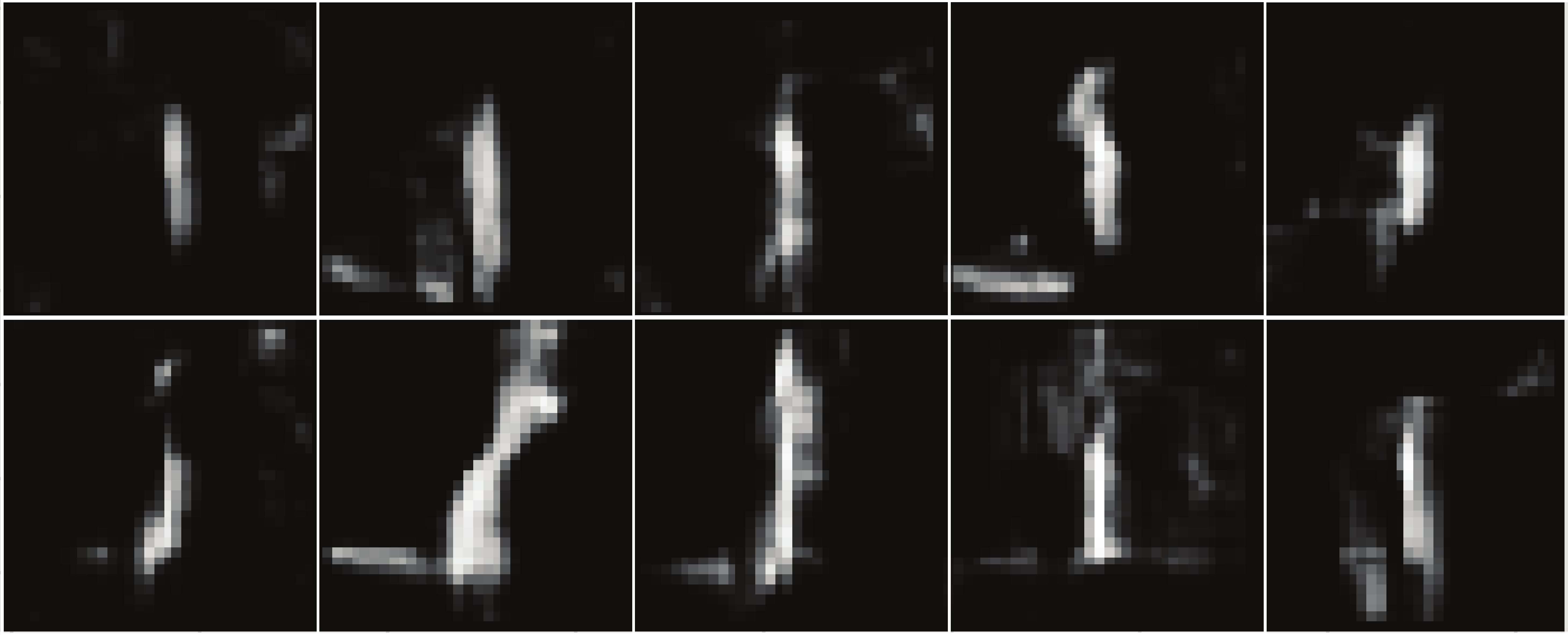}
    \vspace{-2mm}
        \caption{ \textbf{Sampled ``1'' images by a single generator.} 
        Based on the observation we draw from \figref{totalweightsratio},
        we train a single GAN using 60k Fashion-MNIST images together with $300$ MNIST ``1''
        images, and the GAN is able to generate images close to ``1''. 
        Here we show a few generated images from the resulting generator.
    }\label{fig:single_DC_highconf}
\end{figure}

\paragraph{Mixture of Gaussians and Stacked MNIST.}  
We conduct experiments on the same synthetic dataset and Stacked MNIST as used in AdaGAN~\cite{tolstikhin2017adagan}.
All synthetic data points are distributed on a 2D plane, consisting of $M$ ($M=10$)
Gaussians uniformly sampled in a squared region $[-15,15]\times[-15,15]$, 
all with the same variance $\sigma_0^2=0.05$. 

We evaluate our algorithm by checking how many iterations (i.e., the parameter
$T$ in \algref{mwu_training}) it takes to cover all modes, and compare it with AdaGAN.
A mode is considered covered, if in $N$ generated samples, there exist at least
$0.01\cdot N/M$ samples landed within a distance $3\sigma_0$ away from the mode's center.
The experiments on both our algorithms and AdaGAN are repeated 15 times. 
On this synthetic
dataset, both our algorithm and AdaGAN can cover all modes in 2 iterations.
For Stacked MNIST, both our method and AdaGAN can cover all modes in 5 iterations.


\paragraph{More challenging synthetic datasets.}
Furthermore, we test our method and AdaGAN on two other synthetic datasets that have more challenging mode
distributions.  
The first one, referred as $\mathsf{D}_s$, has 20 modes
distributed along a spiral curve (see \figref{spiral}-left). 
Each mode is a set of points following a Gaussian distribution (with a variance of 1).
The center of $i$-th mode ($i=1..20$) is located at
$(\cos(i / 3) \cdot i \cdot i, \sin(i / 3) \cdot i \cdot i)$. 
The second dataset, referred as $\mathsf{D}_i$, has $21\times 21 + 1$ modes, among which
$21\cdot 21=441$ modes locate on a $[-10,10]\times[-10,10]$ uniform grid and
one additional mode is isolated at $(100,100)$ (see \figref{spiral}-right). 
Each mode is also a set of points under a Gaussian distribution (with a variance of 0.05).

For both datasets, we evaluate how many modes are covered as the number of 
iterations increases in both our method and AdaGAN.  The mode coverage is defined in the
same way as in the previous experiment.  As shown in \figref{spiral}, our
algorithm covers all the modes, and outperforms AdaGAN on both datasets.
In terms of efficiency, AdaGAN takes 437 min (25 iterations) and still miss
some modes, while our method takes only 134 min (9 iterations) to cover all modes.

%
%
%
%

\end{document}